%% file: main.tex
\pdfoutput=1

\documentclass[onefignum,onetabnum]{siamonline190516}

\usepackage{enumerate}
\usepackage{silence}
\input{shared}

\DeclarePairedDelimiter\abs{\lvert}{\rvert}
\makeatletter
\let\oldabs\abs
\def\abs{\@ifstar{\oldabs}{\oldabs*}}

\begin{document}

\maketitle

\input{main_text}

\bibliographystyle{siamplain}
\bibliography{tidy}

\appendix
\input{supplement_text}
\end{document}

%% file: shared.tex
\usepackage{algorithmic}

\usepackage{amsfonts}
\usepackage{graphicx}
\usepackage{epstopdf}

\ifpdf
  \DeclareGraphicsExtensions{.eps,.pdf,.png,.jpg}
\else
  \DeclareGraphicsExtensions{.eps}
\fi

\usepackage{amsopn}
\usepackage{amsmath}
\usepackage{amssymb}
\usepackage{comment}
\usepackage{dsfont}
\usepackage{mathtools}
\usepackage{nicefrac}
\usepackage{paralist}
\usepackage{tikz}
\usepackage{pgfplots}
\usepackage{xspace}
\usepackage{xcolor}
\usepackage{subcaption}

\definecolor{bleu}     {RGB}{ 49,140,231}
\definecolor{cardinal} {RGB}{196, 30, 58}
\definecolor{emerald}  {RGB}{ 80,200,120}
\definecolor{lightgrey}{RGB}{230,230,230}
\definecolor{myorange}{RGB}{230,159,0}
\definecolor{myskyblue}{RGB}{86,180,233}
\definecolor{mybluegreen}{RGB}{0,158,115}
\definecolor{myyellow}{RGB}{240,228,66}

\usepgfplotslibrary{colorbrewer}
\usepgfplotslibrary{groupplots}

\pgfplotsset{%
  compat                 = 1.17,
  filter discard warning = false, 
  discard if not/.style 2 args={
    x filter/.append code={
      \edef\tempa{\thisrow{#1}}
      \edef\tempb{#2}
      \ifx\tempa%
        \tempb%
      \else%
        
      \fi
    }
  },
}

\pgfplotsset{%
  lineplot/.style = {%
    axis x line*      = bottom,
    axis y line*      = left,
    tick align        = outside,
    legend cell align = left,
    legend pos        = south east,
    legend columns    = -1,
    legend style      = {%
      draw         = lightgrey,
      font         = \fontsize{4}{5}\selectfont,
      at           = {(0.01, 0.02)},
      anchor       = south west,
    },
    legend image post style = {%
      scale = 0.5
    },
    grid              = none,
    major grid style  = {%
      lightgrey,
      thin
    },
    yticklabel style  = {
      /pgf/number format/precision=2,
      /pgf/number format/fixed,
      /pgf/number format/fixed zerofill,
    },
    thick,
    enlarge x limits  = 0.02,
    enlarge y limits  = 0.02,
    xmin              = 0.0,
    xmax              = 1.0,
    xlabel            = {$t$},
    height            = 3.75cm,
    width             = 0.90\linewidth,
    cycle list/Dark2,
    /tikz/font = {\small},
  }
}

\DeclarePairedDelimiter{\norm}{\lVert}{\rVert} 
\newcommand{\Pt}{\mathbf{P}_{t}}

\newcommand\review[1]{\textcolor{black}{#1}} 
\newcommand\reviewtwo[1]{\textcolor{black}{#1}} 

\newcommand{\betti}[1]                  {\ensuremath{\beta_{#1}}}
\newcommand{\boundaryop}[1]             {\ensuremath{\partial_{#1}}\xspace}
\newcommand{\chaingroup}[1]             {\ensuremath{C_{#1}}\xspace}
\newcommand{\diagram}                   {\ensuremath{\mathcal{D}}\xspace}
\newcommand{\homologygroup}[1]          {\ensuremath{H_{#1}}\xspace}

\newcommand{\naturals}                  {\ensuremath{\mathds{N}}\xspace}
\newcommand{\persistentbetti}        [2]{\ensuremath{\betti{#1}^{(#2)}}}
\newcommand{\persistenthomologygroup}[2]{\ensuremath{\homologygroup{#1}^{(#2)}}}
\newcommand{\reals}                     {\ensuremath{\mathds{R}}\xspace}
\newcommand{\simplicialcomplex}         {\ensuremath{\mathrm{K}}\xspace}

\DeclareMathOperator*{\argmin}          {arg\,min}
\DeclareMathOperator{\conv}             {conv}

\DeclareMathOperator{\diam}             {diam}
\DeclareMathOperator{\dist}             {d}
\DeclareMathOperator{\bottleneck}       {\dist_{b}}  
\DeclareMathOperator{\dhausdorff}       {\dist_{H}}
\DeclareMathOperator{\dgromovhausdorff} {\dist_{GH}}
\DeclareMathOperator{\im}               {im}

\DeclareMathOperator{\persistence}      {pers}
\DeclareMathOperator{\rank}             {rank}
\DeclareMathOperator{\vietoris}         {\mathcal{V}}

\newcommand{\mP}{\mathbf{P}}
\newcommand{\mK}{\mathbf{K}}
\newcommand{\mA}{\mathbf{A}}
\newcommand{\mD}{\mathbf{D}}
\newcommand{\X}{\mathsf{X}}
\newcommand{\Xt}{\mathsf{X}_t}
\newcommand{\Kt}{\mathbf{K}_t}
\newcommand{\inP}[1]{\mathbf{P}^{(#1)}}
\newcommand{\interior}{\mathrm{int}}

\usepackage[shortlabels]{enumitem}
\setlist[enumerate]{leftmargin=.5in}
\setlist[itemize]{leftmargin=.5in}

\newsiamremark{remark}{Remark}
\newsiamremark{hypothesis}{Hypothesis}
\crefname{hypothesis}{Hypothesis}{Hypotheses}
\newsiamthm{claim}{Claim}

\headers{Time-inhomogeneous diffusion geometry and topology}{G. Huguet et al.}

\title{Time-inhomogeneous diffusion geometry and topology\thanks{Submitted to the editors March 28, 2022. $^\star$ Equal contribution. $^{\star\star}$ Equal senior contribution.
\funding{This work was partially funded by IVADO Professor funds, CIFAR AI Chair, and NSERC Discovery grant 03267 [\emph{G.W.}]; NSF grant DMS-1845856~[\emph{M.H.}]; and NIH grant NIGMS-R01GM135929~[\emph{M.H.,G.W.,S.K.}]. The content provided here is solely the responsibility of the authors and does not necessarily represent the official views of the funding agencies.}}}

\author{
Guillaume Huguet$\,^{\star,}$\thanks{Dept. of Math. \& Stat., Universit\'{e} de Montr\'{e}al; Mila - Quebec AI Institute, Montreal, QC, Canada}
\and Alexander Tong$\,^{\star,}$\thanks{Dept. of Comp.\ Sci.\ \& Oper.\ Res., Universit\'{e} de Montr\'{e}al; Mila - Quebec AI Institute, Montreal, QC, Canada}
\and Bastian Rieck$\,^{\star,}$\thanks{Institute of AI for Health, Helmholtz Munich \& Technical University of Munich, Munich, Germany}
\and Jessie Huang$\,^{\star,}$\thanks{Depts. of Comp. Sci. \& Genetics, Yale University, New Haven, CT, USA}
\and Manik Kuchroo$\,$\footnotemark[5]
\and 
Matthew Hirn$\,^{\star\star,}$\thanks{Depts. of CMSE \& Mathematics, Michigan State University, East Lansing, MI, USA}
\and Guy Wolf$\,^{\star\star,}$\footnotemark[2]$\;\,^{,}$\thanks{Correspondence to \email{guy.wolf@umontreal.ca} and \email{smita.krishnawamy@yale.edu}}
\and Smita Krishnaswamy$\,^{\star\star,}$\footnotemark[5]$\;\,^{,}$\footnotemark[7]}


%% file: main_text.tex
\begin{abstract}
    Diffusion condensation is a dynamic process that yields a sequence of multiscale data representations that aim to encode meaningful abstractions. It has proven effective for manifold learning, denoising, clustering, and visualization of high-dimensional data. Diffusion condensation is constructed as a time-inhomogeneous process where each step first computes and then applies a diffusion operator to the data. We theoretically analyze the convergence and evolution of this process from geometric, spectral, and topological perspectives. From a geometric perspective, we obtain convergence bounds based on the smallest transition probability and the radius of the data, whereas from a spectral perspective, our bounds are based on the eigenspectrum of the diffusion kernel. Our spectral results are of particular interest since most of the literature on data diffusion is focused on homogeneous processes. From a topological perspective, we show diffusion condensation generalizes centroid-based hierarchical clustering. We use this perspective to obtain a bound based on the number of data points, independent of their location. To understand the evolution of the data geometry beyond convergence, we use topological data analysis. We show that the condensation process itself defines an intrinsic condensation homology. We use this intrinsic topology as well as the ambient persistent homology of the condensation process to study how the data changes over diffusion time. We demonstrate both types of topological information in well-understood toy examples. Our work gives theoretical insights into the convergence of diffusion condensation, and shows that it provides a link between topological and geometric data analysis. 
\end{abstract}




\section{Introduction}

Graph representations of high-dimensional data have proven useful in
many applications such as visualization, clustering, and denoising.
Typically, a set of data points is described by a graph using a pairwise
affinity measure, stored in an affinity matrix. With this
matrix, one can define the random walk operator or the graph Laplacian,
and use numerous tools from graph theory to characterize the input
data. Diffusion operators are closely related to random walks on
a graph, as they describe how heat (or gas) propagates across the vertices.
Using powers of this
operator yields a time-homogeneous Markov process that has been extensively studied. Most notably, Coifman et
al.~\cite{coifman_diffusion_2006} proved that, under specific conditions, this operator converges to
the heat kernel on an underlying continuous manifold. Manifold learning methods like diffusion maps \cite{coifman_diffusion_2006} define an embedding via
the eigendecomposition of the diffusion operator. Other methods, such as PHATE
\cite{moon_visualizing_2019}, embed a diffusion-based distance by
multidimensional scaling. Various clustering algorithms rely on the
eigendecomposition of this operator (or the resulting Laplacian)
\cite{maggioni2019learning,von2007tutorial}. However, this homogeneous process requires a bandwidth in order to fix and determine the scale of the captured data manifold. If we are interested in considering multiple scales of the data~\cite{brugnone_coarse_2019,kuchroo_multiscale_2020}, or if the data is sampled from a time-varying manifold~\cite{marshall2018time}, we need a time-inhomogenous process.

In this paper, we focus on the time-inhomogeneous diffusion process for
a given initial set of data points. This process is known as
\emph{Diffusion Condensation} \cite{brugnone_coarse_2019} and yields a
representation of the data by a sequence of datasets, each at
a different granularity. This sequence is obtained by iteratively
applying a diffusion operator. It has proven effective for tasks such as
denoising, clustering, and manifold learning
\cite{brugnone_coarse_2019,kuchroo_multiscale_2020,marshall2018time,szlam2008regularization,van2018recovering}.
In this work, we study theoretical questions of diffusion
condensation. Thus, we define conditions on the
diffusion operators such that the process converges to a \emph{single}
point. 
\review{The convergence to a point is a valuable characteristic, as it is a necessary condition for any process that sweeps a complete
range of granularities of the data.} We present this analysis from
a geometric and a spectral perspective, addressing different families
of operators. We also study how the intrinsic shape of the condensed datasets evolves through
condensation time using tools from topological data analysis. In
particular, we define an intrinsic filtration based on the condensation
process, resulting in the notions of persistent and condensation homology, for studying individual condensation steps or for summarizing
the entire process, respectively. Making use of a topological
perspective, we also prove the relation between diffusion condensation
and types of hierarchical clustering algorithms. 

The paper is organized as follows. In \cref{sec:Diffusion Condensation}, we present an overview of diffusion condensation. In \cref{sec:geometric}, we develop a geometric analysis of the process, most importantly we prove its convergence to a point. In \cref{sec:spectral}, we study the convergence of the process from a spectral perspective. In \cref{sec:topological}, we present a topological analysis of the process and relate diffusion condensation to existing hierarchical clustering algorithms.


\section{Diffusion condensation}\label{sec:Diffusion Condensation}



In order to establish the setup and scope for our work, we first formalize here the diffusion condensation framework, and provide a unifying view of design choices and algorithms used to empirically evaluate its efficacy in previous and related work.


\subsection{Notations and setup}\label{sec:dc:notation}

Let $\X = \{x(j) : j=1,\ldots,N\} \subset \reals^d$ be an input dataset of $N$ data points in $d$ dimensions. Given a symmetric nonnegative affinity kernel $k: \reals^d \times \reals^d \rightarrow \reals$, with $0 \le k(x,y) = k(y,x) \le 1$, $x,y \in \reals^d$, we define an $N \times N$ kernel matrix $\mK$ with entries $\mK(i,j) := k(x(i),x(j))$, which can be regarded as a weighted adjacency matrix of a graph capturing the intrinsic geometry of the data. Furthermore, the kernel and resulting graph are often considered as providing a notion of locality in the data, which can be tuned by a kernel bandwidth parameter $\epsilon$. We defer discussion of specific $k$ dependent on $\epsilon$ to \cref{sec:dc:parameters}, but mention that it can be regarded as a proxy for the size or (local) radius of the neighborhoods defined by the kernel.
The diffusion framework for manifold learning~\cite{coifman_diffusion_2006,moon_visualizing_2019} uses this construction to define a Markov process over the intrinsic structure of the data by normalizing the kernel matrix with a diagonal degree matrix $\mD := \mathrm{diag}(d(1), d(2), \ldots, d(N))$ where $d(i) := \sum_j \mK(i,j)$, resulting in a row stochastic Markov matrix $\mP := \mD^{-1} \mK$, known as the (discrete) \emph{diffusion operator}. Traditionally, time-homogeneous diffusion processes leverage powers $\mP^\tau$ of this diffusion operator, for diffusion times $\tau\in\naturals$, to capture underlying data-manifold structure in $\X$ and to organize the data along this structure~\cite{coifman_diffusion_2006, moon_visualizing_2019}. 

Here, on the other hand, we follow the diffusion condensation
approach~\cite{brugnone_coarse_2019} and use a time-inhomogeneous
process, where the diffusion operator (and underlying finite dataset)
vary over time. We consider a sequence of
datasets $\Xt = \{x_t(j) : j=1,\ldots,N\}$, ordered along diffusion
condensation time $t\in\naturals$, with corresponding diffusion
operators $\mP_t$, each constructed over the corresponding $\Xt$. With
a slight abuse of notation we often refer to $\X_t$ as a set or as
an $N\times d$ matrix, where $x_{t}(j)$ is the j-th row or equivalently
the j-th element of the set. At time $t = 0$ we consider the input
dataset, with its (traditional) diffusion operator, while for each $t
> 0$ we take $\Xt := \mP^\tau_{t-1} \X_{t-1}$, with the usual matrix
multiplication. Then, instead of powers of a single diffusion operator,
the $t$-step condensation process is defined via $\mP^{(t-1)} :=
\mP^\tau_{t-1}\dotsm \mP^\tau_0$, and thus we can also directly write
$\X_{t} = \mP^{(t-1)}\X_0$. \review{Note that $\mP^{(t-1)}$ is constructed from a collection of operators based on different datasets, and potentially different bandwidth parameters or kernels, therefore making the process time-inhomegenous.} For simplicity, we keep the diffusion time
$\tau$, but it could also depend on the condensation time $t$. Finally,
we use the notation $\X^{(T)} := \X_0, \X_1, \dots, \X_T$ for
a sequence of datasets up to finite time $T$, and denote the diameter of
the dataset at time $t$ as $\diam(\Xt) := \max_{x,y \in \Xt} \|x
- y\|_2$.

\subsection{Related work using diffusion condensation for data analysis and open questions}\label{sec:dc:related}

The diffusion condensation algorithm first proposed in Brugnone et
al.~\cite{brugnone_coarse_2019} has been applied for data analysis in
a number of areas. Moyle et al.~\cite{moyle_structural_2021} applied
diffusion condensation to study neural connectomics between species and
identify biologically meaningful substructures. Kuchroo et
al.~\cite{kuchroo_multiscale_2020} applied diffusion condensation to
embed and visualize single-cell proteomic data to explore the effect of
COVID-19 on the immune system. Kuchroo et
al.~\cite{kuchroo_topological_2021} applied diffusion condensation on
single-nucleus RNA sequencing data from human retinas with age-related
macular degeneration (AMD) and found a potential drug target by
exploring the topological structure of the resulting diffusion
condensation process. van Dijk et al.~\cite{van2018recovering} applied
one step of diffusion condensation ($T=1$) with high $\tau$ to
single-cell RNA sequencing data to impute gene expression. They showed
that high $\tau$ improves the quality of downstream tasks such as
gene-gene relationships and visualization. These works demonstrate the
empirical utility of diffusion condensation in a number of settings,
specifically when multiscale clustering and visualization is
needed and the data lies on a manifold.

The diffusion condensation process is a particular type of
time-inhomogeneous diffusion process. General time-inhomogeneous
diffusion processes over time-varying data were studied in
\cite{marshall2018time}, where it was proposed
to use the singular value decomposition of the operator $\mP^{(t)}$ to embed an arbitrary
sequence of datasets $\X^{(T)}$ according to their space-time geometry.
Additionally, if those datasets $\X^{(T)}$ were sampled from a manifold
$(\mathcal{M}, g(t))$ with time-varying metric tensor $g(t)$, it was shown in~\cite{marshall2018time} that as $N, T \rightarrow \infty$,
the operator $\mP^{(t)}$ converges to the heat kernel of $(\mathcal{M},
g(t))$. 
We also note a resemblance to the mean shift algorithm
\cite{fukunaga1975estimation,cheng1995mean}, which relies on
a kernel-based estimation of $\nabla \log p(x)$, where $p(x)$ is the
unknown density from which the points are sampled. The processed dataset is 
recursively updated via $x_{t+1}(i) = x_t(i) + \epsilon\nabla \log
p(x)$, which effectively moves all points toward a mode of the
distribution, hence creating clusters. 

Motivated by these empirical successes and inspired by the general theoretical results on time-inhomogeneous diffusion processes, we consider two open questions specific to the diffusion condensation process. Under what conditions does the diffusion condensation algorithm converge? How can the topology of the diffusion condensation process be understood?

\subsection{Theoretical contributions}\label{sec:dc:contributions}

The main contribution of this paper is to
address these open questions and establish the
underpinnings of diffusion condensation. 
Our investigation is divided into three perspectives. First, we
investigate the convergence properties of diffusion condensation under
various parameter regimes from a geometric perspective in
\cref{sec:geometric}, i.e., arrangement of data points in spatial
coordinates. This geometric perspective gives an intuitive sense of
convergence for a large family of kernels with minimum tail bounds.
Next, in \cref{sec:spectral}, we investigate convergence from a spectral
graph theory perspective and prove convergence in terms of the spectral
properties of the kernel, viewing diffusion condensation as
a non-stationary Markov process. A spectral perspective gives bounds in
terms of the eigenvalues of the kernel, which can give better rates of
convergence depending on considered data. Finally, in
\cref{sec:topological}, we investigate the topological characteristics
of diffusion condensation. \review{%
Here we describe both the structure of the dataset
at each condensation step individually via its persistent homology, 
as well as the topology of the condensation
process itself, which we refer to as \emph{condensation homology}.}
Additionally, we link the topology of the diffusion condensation process
to hierarchical clustering and prove how it generalizes centroid
linkage.

\subsection{Algorithm}\label{sec:dc:algorithm}

\begin{algorithm}
\caption{Diffusion Condensation}
\label{alg:condensation}
\begin{algorithmic}[1]
\STATE{Input: Dataset $\X_0$, initial kernel parameter $\epsilon_0$, diffusion time $\tau$, and merge radius $\zeta$}
\STATE{Output: Condensed datasets $\X^{(T)}$}
\FOR{$t \in \{0,1,\dotsc,T-1\}$}
    
    \STATE{$\Kt \gets \textup{kernel}(\Xt, \epsilon_t)$}
    \STATE{$\Pt \gets \mD^{-1}_t \mK_t$}
    \STATE{$\X_{t+1} \gets \Pt^\tau \Xt$}
    \STATE{$\epsilon_{t+1} \gets \textup{update}(\epsilon_t, \X_{t+1})$}
    \FOR{$x_t(i), x_t(j) \in \Xt$}
            \STATE{$\textup{merge}(x_t(i), x_t(j)) \text{ if } \|x_t(i) - x_t(j)\|_2 < \zeta$}\label{lst:Merge event}
    \ENDFOR
\ENDFOR
\STATE{$\X^{(T)} \gets \{\X_0, \X_1, \ldots, \X_T\}$}
\end{algorithmic}
\end{algorithm}

The diffusion condensation algorithm summarizes input data with a series of representations, organized by condensation time, with earlier representations providing low level, microscopic details and later representations providing overall, macroscopic summarizations. Each time step of diffusion condensation can be broken up into five main steps. 
\begin{enumerate}
    \item Construct a kernel matrix $\Kt$ summarizing similarities between points.
    \item Construct a Markov normalized diffusion operator $\Pt$.
    \item Diffuse the data coordinates $\tau$ steps using $\Pt^\tau$.
    \item Update the kernel bandwidth $\epsilon$ according to some $\texttt{update}$ function.
    \item (Optionally) merge points within distance $\zeta$.
\end{enumerate}
\cref{alg:condensation} shows pseudocode for this process.
At each time step, the positions of points are updated based on the
predefined kernel through $\tau$ steps of diffusion. Intuitively, this
can be thought of as moving each point to a kernel-weighted average of
its neighbors, 
The condensation process will behave differently depending on the choice
of kernel, the kernel bandwidth, the diffusion time, and the merging
threshold. 


\review{\cref{fig:tau-motivation} depicts the differences between time-homogenous condensation, time inhomogenous condensation, and a mixture between the two. Greater values of $\tau$ encourage the process to condense along the manifold, in contrast with other hierarchical clustering algorithms that are not able to do so. Comparing only inhomogenous condensation $\mP_{3i}$ (top) with a mixture of homogenous and inhomogenous condensation $\mP^3_i$ (middle) we see that the mixture condenses the moon structures along the manifold rather than shattering them. Both of these are able to separate out the two clusters. In contrast, the fully time-homogenous condensation process $\mP^{3i}$ (bottom) mixes eventually mixes the two moons.} For the rest of the paper, we let $\tau = 1$, but our results are valid for any $\tau\in\naturals$. Only for the spectral part, we need to consider a slight nuance, which we discuss in \cref{rem: spec_tau}.

\begin{figure}[ht]
\begin{center}
\centerline{\includegraphics[width=\columnwidth]{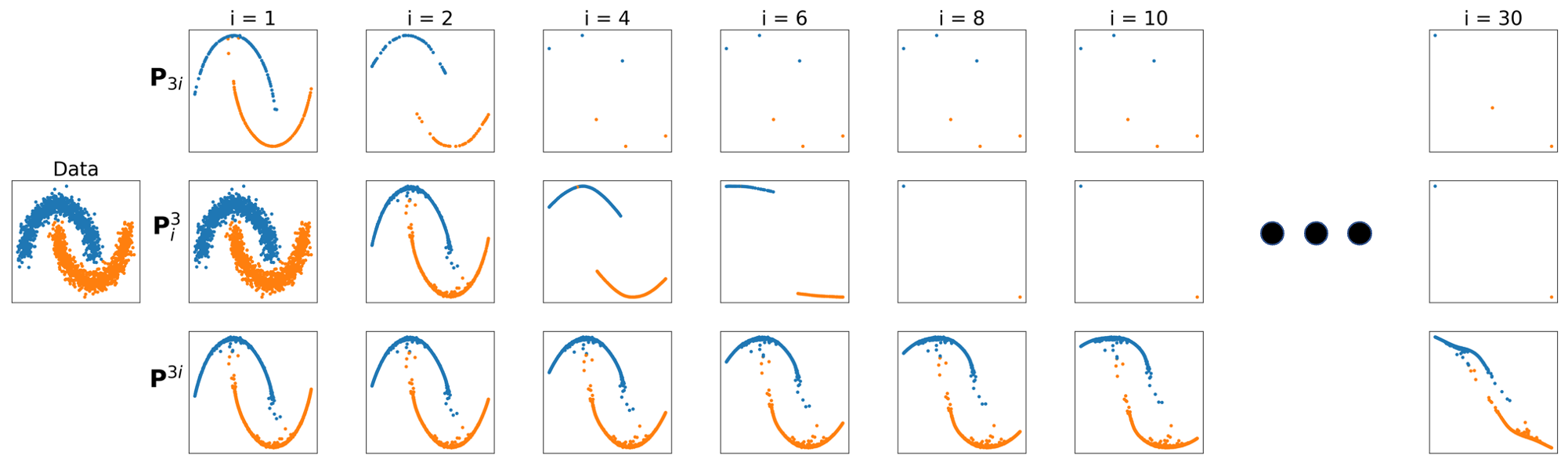}}
\caption{\review{Shows the effects of powering the diffusion operator with a power $\tau$ before condensing on the moons dataset. Shows from top to bottom the time-inhomogenous process with $\tau=1$, with $\tau=3$, and the time-homogenous process with $T = 1$ while varying $\tau$. For step $i$ this corresponds to comparing the applications of $\mP_{3i}$, $\mP^3_{i}$, and $\mP^{3i}$ to the data. $\mP_{3i}$ eventually merges, but has a semi-stable state of 6 points after shattering the moons. $\mP^3_i$ correctly identifies the two clusters of the data efficiently by first condensing along the moons individually into points. Time-homogenous condensation $\mP^{3t}$ mixes the two moons.}}
\label{fig:tau-motivation}
\end{center}
\end{figure}


\begin{remark}
  \review{The \emph{time-homogeneous} equivalent of the condensation process would be to recursively apply the same diffusion operator $\mP_0$ on the initial dataset $\X_0$. After $t$ iterations, the new dataset would simply be $\mP_0^t\X_0$. This constrasts with the \emph{time-inhomegeneous} version where we create a new diffusion operator at each iteration. Concequently, after $t$ iterations, the new dataset is $\mP_{t-1}\dotsc\mP_1\mP_0\X_0$. By using a time-inhomogeneous process, we gain more control over the convergence behavior. Indeed, since we allow for modifications of the diffusion probabilities, we can define a schedule for the parameters that could either promote or slow down the convergence of the process. Here, for simplicity, we assumped $\tau=1$, but the same remark follows for any diffusion time $\tau\in\naturals$.}
\end{remark}

\subsection{Kernels for diffusion condensation}\label{sec:dc:parameters}
Here, we review specific kernel constructions used to study condensation properties in later sections. \review{In \cref{fig:Example kernels}, we present a few iterations of the condensation process, depending on the choice of kernel used to construct the diffusion matrix.}  

\begin{definition}[Box Kernel]\label{def:box} The box kernel of bandwidth $\epsilon$ is 
\begin{equation}
    k_{\epsilon}(x,y) = 
\begin{cases}
1 & \text{ if } \|x-y\|_2 \le \epsilon \\ 
0 & \text{ else}.
\end{cases}
\end{equation}
\end{definition}
The box kernel is arguably the simplest and most interpretable kernel,
leading to interesting data summarizations, including an instance of
agglomerate clustering depending on the bandwidth as a function of time.
However, first we note some simple cases of bandwidth settings. Consider
a case where $k_t(x,y) = 1$ for all $x,y \in \X_0$. This can be
thought of as a box kernel with bandwidth greater than $\diam(\X_0)$.
Using this kernel, after a single step of diffusion condensation,
all points converge to the mean data point, $\frac{1}{n} \sum_i x_t(i)$.
This mean data point is a useful, if trivial, summarization of the
data.
Next, consider the opposite extreme, a box kernel with infinitely narrow
bandwidth, $k_t(x,y) = \{1 \text{ if } x = y \text{ else } 0\}$.
In this case, we have $\X_t = \X_0$ for all $t > 0$, resulting in
another trivial result, i.e., \emph{no} data summarization over
diffusion condensation time. Of more interest are bandwidths between
these two extremes, providing hierarchical sets of summarizations. 
%
%
Next, we consider smoother kernels.

\begin{definition}
\label{def: alpha-decay Kernel} The \emph{$\alpha$-decay kernel}~\cite{moon_visualizing_2019} of bandwidth $\epsilon$ is 
$ 
k_{\epsilon, \alpha}(x,y) = \exp(-\| x - y \|_2^\alpha / \epsilon^\alpha).
$ 
\end{definition}
The $\alpha$-decay kernel was used in \cite{kuchroo_multiscale_2020}
along with anisotropic density normalization~(see \cref{def: Anisotropic
Kernel}), which was shown to empirically speed up convergence of
diffusion condensation.

\begin{definition} 
\label{def: Gaussian_kernel} The \emph{Gaussian kernel} of bandwidth $\epsilon$ is 
$ 
k_\epsilon(x,y) = \exp(-\|x - y\|_2^2 / \epsilon).
$ 
\end{definition}
The Gaussian kernel was used in~\cite{brugnone_coarse_2019}, employing
density normalization and a merging threshold of $10^{-4}$, with
a bandwidth of $\epsilon_t$ doubling whenever the change in position of
points between $t-1$ and $t$ dropped below a separate threshold.
This kernel and setting of $\epsilon_t$ ensures that the datasets
converge to a single point in a reasonable amount of time in practice.

Another kernel that exhibits interesting behavior is the Laplace kernel; it is noteworthy since it is
positive definite for all conditionally negative definite
metrics~\cite{Feragen15a}.

\begin{definition}
\label{def: Laplacian Kernel} The \emph{Laplace kernel} of bandwidth $\epsilon$ is
$ 
    k_{\epsilon}(x,y) = \exp(-\| x - y \|_2 / \epsilon).
$ 
\end{definition}
Note that this is the same as the $\alpha$-decay kernel, with $\alpha=1$.
In fact, the Gaussian and Laplace kernels can be generalized to the $\alpha$-decay kernel, which interpolates between the Gaussian kernel when $\alpha=2$ and the box kernel as $\alpha \rightarrow \infty$.

\begin{definition}[Anisotropic Density Normalized Kernel~\cite{coifman_diffusion_2006}]\label{def: Anisotropic Kernel}
For a rotation invariant kernel $k_\epsilon(x,y)$, let
$ 
    q(x) = \int_\X k_\epsilon(x,y) q(y) \mathrm{d}y;
$ 
then a density normalized kernel with normalization factor $\beta$ is given by
$ 
k_{\epsilon, \beta}(x,y) = \frac{k_\epsilon(x,y)}{q^\beta(x) q^\beta(y)}.
$ 
\end{definition}


\begin{figure}[tbp]
  \centering
  \begin{tikzpicture}
    \begin{groupplot}[%
      group style = {%
        group size     = 5 by 4,
        ylabels at     = edge left,
        horizontal sep = 0cm,
        vertical sep   = 0cm,
      },
      unit vector ratio* =  1 1 1,
      %
      xmin = -1.5,
      xmax =  1.5,
      ymin = -1.5,
      ymax =  1.5,
      width = 4cm,
      ticks = none,
    ]

    \pgfplotsset{%
      /pgfplots/group/every plot/.append style = {%
        mark size = 0.5pt,
      },
    }
    \nextgroupplot[ylabel = {Gaussian}]
      \addplot[only marks] table {./Data/Petals/petals_gaussian_n128_t00.txt};

      \nextgroupplot
      \addplot[only marks] table {./Data/Petals/petals_gaussian_n128_t01.txt};

      \nextgroupplot
      \addplot[only marks] table {./Data/Petals/petals_gaussian_n128_t05.txt};

      \nextgroupplot
      \addplot[only marks] table {./Data/Petals/petals_gaussian_n128_t15.txt};

      \nextgroupplot
      \addplot[only marks] table {./Data/Petals/petals_gaussian_n128_t20.txt};

      \nextgroupplot[ylabel = {Laplace}]
      \addplot[only marks] table {./Data/Petals/petals_laplacian_n128_t000.txt};

      \nextgroupplot
      \addplot[only marks] table {./Data/Petals/petals_laplacian_n128_t003.txt};

      \nextgroupplot
      \addplot[only marks] table {./Data/Petals/petals_laplacian_n128_t005.txt};

      \nextgroupplot
      \addplot[only marks] table {./Data/Petals/petals_laplacian_n128_t010.txt};

      \nextgroupplot
      \addplot[only marks] table {./Data/Petals/petals_laplacian_n128_t300.txt};

      \nextgroupplot[ylabel = {Box}]
      \addplot[only marks] table {./Data/Petals/petals_box_n128_t00.txt};

      \nextgroupplot
      \addplot[only marks] table {./Data/Petals/petals_box_n128_t05.txt};

      \nextgroupplot
      \addplot[only marks] table {./Data/Petals/petals_box_n128_t07.txt};

      \nextgroupplot
      \addplot[only marks] table {./Data/Petals/petals_box_n128_t10.txt};

      \nextgroupplot
      \addplot[only marks] table {./Data/Petals/petals_box_n128_t13.txt};

      \nextgroupplot[ylabel = {$\alpha$-decay}]
      \addplot[only marks] table {./Data/Petals/petals_alpha_n128_t00.txt};

      \nextgroupplot
      \addplot[only marks] table {./Data/Petals/petals_alpha_n128_t08.txt};

      \nextgroupplot
      \addplot[only marks] table {./Data/Petals/petals_alpha_n128_t13.txt};

      \nextgroupplot
      \addplot[only marks] table {./Data/Petals/petals_alpha_n128_t20.txt};

      \nextgroupplot
      \addplot[only marks] table {./Data/Petals/petals_alpha_n128_t35.txt};

    \end{groupplot}
  \end{tikzpicture}
  \caption{%
    An example of different kernels~(rows) and how they affect convergence
    behavior for the ``petals'' dataset. Convergence speed is highest in
    the box kernel~(15 iterations), followed by the Gaussian kernel~(25
    iterations), and the $\alpha$-decay kernel~(40 iterations, \review{$\alpha = 10.0$}),
    whereas the Laplace kernel requires 491 iterations to converge.
  }
  \label{fig:Example kernels} 
\end{figure}
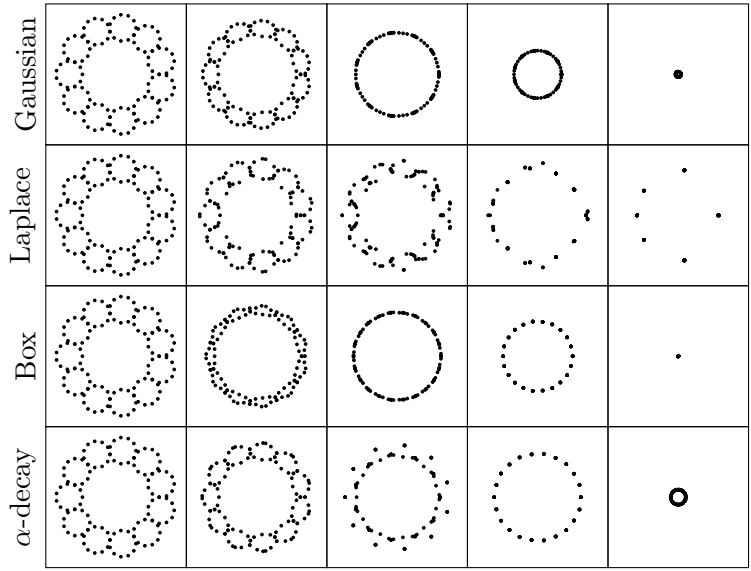

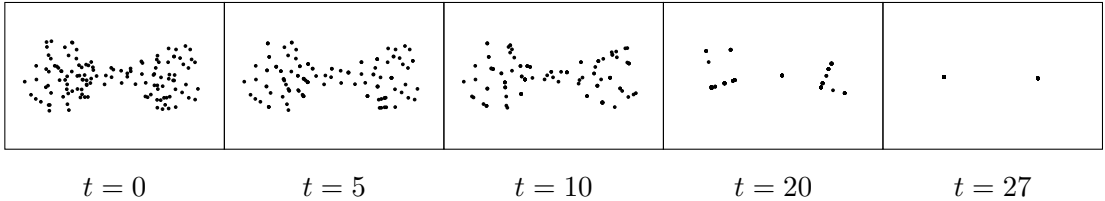
\begin{figure}[tbp]
  \centering
  \begin{tikzpicture}
    \begin{groupplot}[%
      group style = {%
        group size     = 6 by 1,
        ylabels at     = edge left,
        horizontal sep = 0cm,
        vertical sep   = 0cm,
      },
      unit vector ratio* =  1 1 1,
      %
      xmin = -0.25,
      xmax =  2.75,
      ymin = -0.50,
      ymax =  1.50,
      width = 4.5cm,
      ticks = none,
    ]
      \pgfplotsset{%
        /pgfplots/group/every plot/.append style = {%
          mark size = 0.5pt,
        },
        every axis title/.append style ={%
          at = {(0.5,-0.5)},
        },
      }
      \nextgroupplot[title = {$t = 0$}]
      \addplot[only marks] table {./Data/Barbell/barbell_box_n128_t00.txt};
      \nextgroupplot[title = {$t = 5$}]
      \addplot[only marks] table {./Data/Barbell/barbell_box_n128_t05.txt};
      \nextgroupplot[title = {$t = 10$}]
      \addplot[only marks] table {./Data/Barbell/barbell_box_n128_t10.txt};
      \nextgroupplot[title = {$t = 20$}]
      \addplot[only marks] table {./Data/Barbell/barbell_box_n128_t20.txt};
      \nextgroupplot[title = {$t = 27$}]
      \addplot[only marks] table {./Data/Barbell/barbell_box_n128_t27.txt};
    \end{groupplot}
  \end{tikzpicture}
  \caption{%
    \review{%
      Convergence behavior of a ``dumbbell'' dataset for different
      iterations of a box kernel with fixed bandwidth. The process
      converges to two points that are not connected.
    }
  }
  \label{fig:Example dumbbell} 
\end{figure}

\begin{figure}[tbp]
  \centering
  \begin{tikzpicture}
    \begin{groupplot}[%
      group style = {%
        group size     = 6 by 1,
        ylabels at     = edge left,
        horizontal sep = 0cm,
        vertical sep   = 0cm,
      },
      unit vector ratio* =  1 1 1,
      %
      xmin = -1.5,
      xmax =  1.5,
      ymin = -1.5,
      ymax =  1.5,
      width = 4cm,
      ticks = none,
    ]

    \pgfplotsset{%
      /pgfplots/group/every plot/.append style = {%
        mark size = 0.5pt,
      },
    }
    \nextgroupplot
      \addplot[only marks] table {./Data/Hyperuniform_Circle/hyperuniform_circle_gaussian_n128_t01.txt};

      \nextgroupplot
      \addplot[only marks] table {./Data/Hyperuniform_Circle/hyperuniform_circle_gaussian_n128_t10.txt};

      \nextgroupplot
      \addplot[only marks] table {./Data/Hyperuniform_Circle/hyperuniform_circle_gaussian_n128_t20.txt};

      \nextgroupplot
      \addplot[only marks] table {./Data/Hyperuniform_Circle/hyperuniform_circle_gaussian_n128_t30.txt};

      \nextgroupplot
      \addplot[only marks] table {./Data/Hyperuniform_Circle/hyperuniform_circle_gaussian_n128_t40.txt};

      \nextgroupplot
      \addplot[only marks] table {./Data/Hyperuniform_Circle/hyperuniform_circle_gaussian_n128_t50.txt};

    \end{groupplot}
  \end{tikzpicture}
  \caption{%
    The convergence behavior of a hyperuniform circle\review{---a circle with equally spaced points around its circumference---}using a Gaussian
    kernel.
  }
  \label{fig:Hyperuniform circle convergence} 
\end{figure}
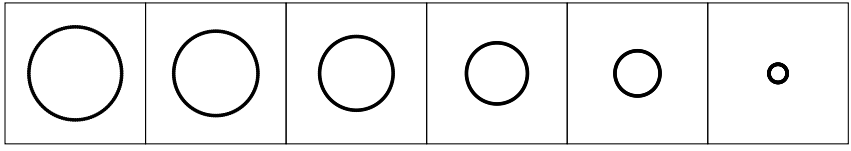

\section{Geometric properties of the condensation process}\label{sec:geometric}

We first examine the time-varying nature of the data geometry along the
diffusion condensation process. Since this process results in a sequence
of finite datasets $\Xt$, organized along condensation time, a pertinent
question is how does their underlying geometric structure change as
local variability is eliminated by the diffusion process, and whether
it eventually converges to a stable one as $t\to\infty$. In this
section, we study this question by considering two geometric
characteristics (namely, convex hull and diameter) of the data,
establish their monotonic convergence, and its relation to the tail
behavior of the kernel utilized in the construction of the diffusion
process.


Our main result here is that with appropriate kernel choice the
condensation process converges to a point, in time dependent on the
shape of the kernel. That is, for all $\zeta > 0$, there exists a $M \in
\naturals$ such that for all $t\geq M$, we have $\norm{x-y}<\zeta$ for
all $x,y \in \Xt$. Intuitively, we can make all the points arbitrarily
close by iterating the process. It is important to note that this result
requires some assumptions on the kernel in order to avoid pathological
cases where the process may converge, but not to a single point. For
instance, using the box kernel on a dumbbell dataset, each sphere would
converge to a point, but for a certain threshold these points would not
be connected. Hence, the process would reach a stable state, i.e., there
exists $M\in\naturals$ such that $\mP_M\X_M = \X_M$, but it would not
converge to a point~\review{(see \cref{fig:Example dumbbell} for an
illustration)}.
One of our goals is to define the conditions on the kernels for the process to converge to a single point.



\subsection{Diameter and convex hull convergence}

Intuitively, one can consider each diffusion condensation iteration as
eliminating local variability in
data~\cite{brugnone_coarse_2019,kuchroo_multiscale_2020}, and while
empirical results presented in previous work indicate the condensation
process can accentuate separation between weakly connected data regions,
they also indicate that the process has a global contraction property due to
the elimination of variability in the data. Thus, it appears that diffusion
condensation coarse-grains data by sweeping through granularities, from
each point being a separate entity to all data points being in a single
cluster. To establish this contractive property, and formulate a notion
of data geometry (monotone) convergence associated with it, we
characterize the geometry of each $\Xt$ via its diameter and convex
hull, whose convergence under the condensation process is shown in the
following theorem.

\begin{remark}
The diffusion condensation process is not a contractive mapping in the strict sense. During individual iterations, distances are not generally \emph{all} decreasing, i.e., there exist points $x_t(i)$ and $x_t(j)$ such that $\|x_t(i)- x_t(j)\|_2 < \|x_{t+1}(i) - x_{t+1}(j)\|_2$.
\end{remark}


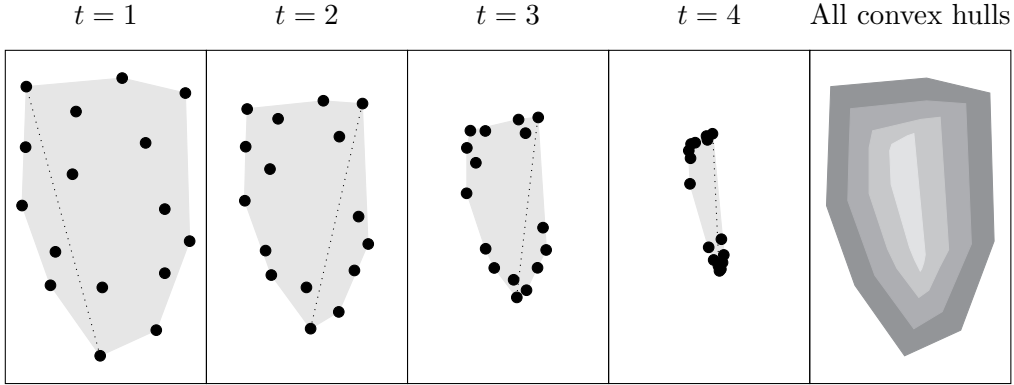
\begin{figure}[tbp]
  \centering
  \begin{tikzpicture}
    \begin{groupplot}[%
     group style = {%
        group size     = 5 by 1,
        horizontal sep = 0.0cm,
        vertical sep   = 0.0cm,
      },
      ticks              = none,
      unit vector ratio* = 1 1 1,
      xmin               = 0.49,
      xmax               = 2.85,
      ymin               = 0.60,
      ymax               = 4.50,
      height             = 6.0cm,
      enlargelimits      = true,
    ]
      \nextgroupplot[title = {$t = 1$}]

      \addplot[only marks] table{Data/Point_Cloud/point_cloud_t00.txt};

      \filldraw[lightgrey] (0.55,4.38) -- (0.49,2.71) -- (0.89,1.59) -- (1.59,0.60) -- (2.38,0.96) -- (2.85,2.21) -- (2.79,4.29) -- (1.90,4.50) -- cycle;
      \draw[dotted] (1.59,0.60) -- (0.55,4.38);

      \nextgroupplot[title = {$t = 2$}]

      \addplot[only marks] table{Data/Point_Cloud/point_cloud_t01.txt};

      \filldraw[lightgrey] (1.90,4.18) -- (0.83,4.07) -- (0.81,3.54) -- (0.79,2.78) -- (1.17,1.73) -- (1.72,0.98) -- (2.12,1.22) -- (2.53,2.17) -- (2.45,4.14) -- cycle;
      \draw[dotted] (1.72,0.98) -- (2.45,4.14);

      \nextgroupplot[title = {$t = 3$}]

      \addplot[only marks] table{Data/Point_Cloud/point_cloud_t02.txt};

      \filldraw[lightgrey] (2.21,2.09) -- (2.09,3.95) -- (1.82,3.92) -- (1.14,3.76) -- (1.09,3.52) -- (1.09,2.88) -- (1.35,2.10) -- (1.48,1.83) -- (1.79,1.42) -- (1.93,1.52) -- (2.08,1.84) -- cycle;
      \draw[dotted] (1.79,1.42) -- (2.09,3.95);

      \nextgroupplot[title = {$t = 4$}]

      \addplot[only marks] table{Data/Point_Cloud/point_cloud_t03.txt};

      \filldraw[lightgrey] (1.88,2.02) -- (1.72,3.72) -- (1.63,3.69) -- (1.41,3.57) -- (1.38,3.48) -- (1.40,3.01) -- (1.66,2.12) -- (1.73,1.95) -- (1.81,1.79) -- (1.83,1.82) -- (1.86,1.91) -- cycle;
      \draw[dotted] (1.81,1.79) -- (1.72,3.72);

      \nextgroupplot[title = {All convex hulls}]

      \filldraw[gray!100] (0.55,4.38) -- (0.49,2.71) -- (0.89,1.59) -- (1.59,0.60) -- (2.38,0.96) -- (2.85,2.21) -- (2.79,4.29) -- (1.90,4.50) -- cycle;
      \filldraw[gray!75 ] (1.90,4.18) -- (0.83,4.07) -- (0.81,3.54) -- (0.79,2.78) -- (1.17,1.73) -- (1.72,0.98) -- (2.12,1.22) -- (2.53,2.17) -- (2.45,4.14) -- cycle;
      \filldraw[gray!50 ] (2.21,2.09) -- (2.09,3.95) -- (1.82,3.92) -- (1.14,3.76) -- (1.09,3.52) -- (1.09,2.88) -- (1.35,2.10) -- (1.48,1.83) -- (1.79,1.42) -- (1.93,1.52) -- (2.08,1.84) -- cycle;
      \filldraw[gray!25 ] (1.88,2.02) -- (1.72,3.72) -- (1.63,3.69) -- (1.41,3.57) -- (1.38,3.48) -- (1.40,3.01) -- (1.66,2.12) -- (1.73,1.95) -- (1.81,1.79) -- (1.83,1.82) -- (1.86,1.91) -- cycle;

    \end{groupplot}
  \end{tikzpicture}
  \caption{%
    Illustration of \cref{thm:diameter_conv}. We depict $4$ time steps
    of the condensation process of a simple dataset. The convex hull
    of the points is shown in gray, with the diameter of $\Xt$ being
    shown as dotted line. As $t$ progresses, convex hulls shrink, with 
    $\conv (\X_{t+1}) \subsetneq \conv(\X_t)$. The rightmost figure
    shows all convex hulls of all time steps with lighter shades
    indicating later condensation time steps.
  }
  \label{fig:Convex hull}
\end{figure}

\begin{theorem}
\label{thm:diameter_conv}
Let $(\Xt)_{t\in\naturals}$ and $(\Pt)_{t\in\naturals}$ be respectively the sequence of datasets and diffusion operators generated by diffusion condensation. If the kernel used to construct each $\Pt$ is strictly (pointwise) positive, then:
\begin{enumerate}
\item Their convex hulls form a nested sequence with 
\[
\lim_{t \to \infty} \conv(\Xt) = \bigcap\limits_{t=1}^{\infty} \conv(\Xt) \neq \emptyset \text{ and convex}.
\]
\item The diameters form a convergent monotonically decreasing sequence with
\[
\lim_{t \to \infty} \diam(\Xt) = \inf_{t\geq1} \diam(\Xt) \geq 0.
\]
Further, $\diam(\X_{s+1})=\diam(\X_s)$ if and only if $\diam(\X_s) = 0$, i.e., for $s\in\naturals$ such that $\diam(\X_s)>0$, we have $\diam(\X_{s+1})<\diam(\X_{s})$.
\item If there exists $k\in\naturals$ such that $\mP_k \X_k = \X_k$, then $\conv(\X_k) = \{x_k\}$, i.e., the process converged to a single point. 
\end{enumerate}
\end{theorem}

\noindent Prior to proving \cref{thm:diameter_conv}, we first require
the following technical lemma about polytopes, which can also be found
in standard literature~\cite{lay2007convex}.
Its proof is provided in the supplementary material for completeness.
\begin{lemma}
  \label{lem:geo_general_conv}
  Let $\X \subset \reals^d$ be a set of points and $C :=
  \conv(\X)$ their convex hull. Then every \emph{extremal
  point} $v_j \in C$ satisfies $v_j \in \X$. Thus, the extremal points
  of~$C$ are a subset of $\X$.
  \label{lem:Vertices of a polytope}
\end{lemma}

\noindent With \cref{lem:geo_general_conv}, we are now ready to prove
\cref{thm:diameter_conv} as follows~(see \cref{fig:Convex hull} for
an illustration of the arguments in the proof).

\begin{proof}[Proof of \cref{thm:diameter_conv}]
\review{
Denote the interior of the convex hull of $\X$ as $\interior(\conv(\X))$.} We start by proving
\begin{enumerate}[(a)]
    \item \review{ if $\diam(\X_t)>0$, then $\conv(\X_{t+1})\subseteq \interior( \conv(\X_t))$, }
    \item $\diam(\X_t)=0$ if and only if $\mP_t \X_t = \X_t$.
\end{enumerate}
To prove (a), we note that since the entries of $\Pt$ are positive and its rows sum to~$1$, each element of $\X_{t+1}$ is a \emph{convex combination} of the original data points. That is, $x_{t+1}(i) = \mP_t(i,\cdot)\X_t$, with $\mP_t(i,j) > 0$ for all $j\in\{1,\dotsc,N\}$, and $\sum_j \mP_t(i,j) = 1$. As a consequence, $\X_{t+1}$ will be formed by convex combinations, so all points in $\X_{t+1}$ lie in the \emph{interior} of $\conv(\X_t)$, which is not empty since $\diam(\X_t)>0$. From \cref{lem:Vertices of a polytope}, we know that the extremal points of $\conv(\X_{t+1})$ also lie in the interior of $\conv(\X_t)$. Hence $\conv (\X_{t+1}) \subseteq \interior(\conv(\X_t))$, which proves (a). To prove (b), we assume $\diam(\X_t)=0$. By construction of $\mP_t$, we get $\mP_t\X_t = \X_t$. Now if we assume $\mP_t\X_t = \X_t$,  we have $\conv(\X_{t+1}) = \conv(\mP_t\X_t) = \conv(\X_t)$. If $\diam(\X_t)>0$, this would contradict (a), hence $\diam(\X_t)=0$. 
Steps (a) and (b) show that the convex hulls are a nested sequence, so
$\conv(\X_t) \to \cap_{t=1}^\infty \conv(\X_t)$. Since the intersection
of convex sets is convex, the limiting set is also convex.
Finally, we use Helly's theorem \cite{lay2007convex}, \review{which
states that if an infinite collection of compact convex subsets in $\reals^d$
has a nonempty intersection for every~$d+1$ subsets, then
the collection of \emph{all} subsets has a nonempty intersection.}
Here, because of the nesting property, every subcollection has nonempty
intersection. Moreover, the convex hulls of finite sets are compact,
hence we conclude that $\cap_{t=1}^\infty\conv(\X_t)$ is not empty.

Finally, (a) and (b) imply $\diam(\X_{t+1})<\diam(\X_t)$ if
$\diam(\X_t)>0$, and $\diam(\X_{t+1})=\diam(\X_t)$ if $\diam(\X_t)=0$.
Thus, the diameters form a monotonically decreasing sequence, which
converges since the sequence is nonnegative. 
\end{proof}


\subsection{Convergence rates}\
\cref{thm:diameter_conv} applies for all strictly positive kernels.
While it is only established in terms of the diameter and convex hull of
the data, this result extends to show pointwise convergence of the
diffusion condensation process if we make further assumptions on the
rate at which the diameter sequence decreases, or establish bounds on
this rate based on the specific kernel used in the diffusion
construction. We next proceed with such an in-depth analysis,
focusing on strictly positive kernels, while noting that later, in
\cref{sec:clustering}, we will also show a convergence result for
a kernel with finite support (i.e., where the discussion here is not
valid). 
We begin with the following result relating the rate of convergence to the minimum value of the kernel over the data.
\begin{lemma}\label{lemma:shrinkrate}
If there exists a nonnegative constant $\delta$, such that $0< \delta \le \mathbf{K}_t(i,j) \le 1$ for all $ t\in\naturals$, then the diameter sequence $(\diam(\X_t))_{t\in\naturals}$ decreases at a speed of at least $1-\delta$, i.e., $\diam(\X_{t+1}) \le (1 - \delta) \diam(\X_t)$. 
\end{lemma}
\begin{proof}
Here we present the key ideas of the proof, and we refer to the supplementary material for the detailed version. 
The assumption on $\mK_t$ gives the element-wise lower bound $\mP_t \geq \delta/N$, and we show that $d_{TV}(\mP_t(i,\cdot),\mP_t(j,\cdot))\leq 1-\delta$, where $d_{TV}$ is the total variation distance. Next, using a coupling $\xi$ with marginals $\mP_t(i,\cdot)$ and $\mP_t(j,\cdot)$, we can write
\begin{align*}
       \| x_{t+1}(i) - x_{t+1}(j)\|_2 & = \|(\Pt(i,\cdot)-\Pt(j,\cdot)) \Xt\|_2 =\| \sum_{i, j} \xi(i, j)(x_t(i) -x_t(j) )\|_2 \\
       & \leq\sum_{i, j} \xi(i, j) \|x_t(i) -x_t(j)\|_2 \le \sum_{i \neq j} \xi(i, j) \diam(\X_t). 
\end{align*}
We conclude with the coupling lemma, which guarantees the existence of a coupling such that $\sum_{i\neq j}\xi(i,j) = d_{TV}(\mP_t(i,\cdot),\mP_t(j,\cdot))$, thus $\diam(\X_{t+1}) \le (1 - \delta) \diam(\X_t)$.
\end{proof}
Given this result for general kernels whose tails can be lower bound by some constant, we can further state a union bound result on kernels that maintain this lower bound over the entire diffusion condensation process. 
We recall at this point the $\epsilon$ update step typically used to
expedite the condensation process when it reaches a slow contraction
meta-stable state. Previous work implemented this step with heuristics
for updating the $\epsilon$ meta-parameter. Here, we provide further
insights into the impact of this update step, to both justify it and
suggest an update schedule that provides certain convergence guarantees
via the following theorem.
\begin{theorem}
\label{thm:geo_rate}
  For some nonnegative constant $\delta$, if there exists an $\epsilon_t$ schedule such that $0< \delta\le \mathbf{K}_t(i,j) \le 1$ for all $ t\in\naturals$, then for any merge threshold $\zeta > 0$, diffusion condensation converges to a single point in $t^* = \left \lceil \frac{\log(\zeta) - \log(\diam(\X_0))}{\log (1 - \delta)} \right \rceil$ steps. 
\end{theorem}
\begin{proof}
  Repeated application of \cref{lemma:shrinkrate} yields $\diam(\X_{t+1}) \le (1 - \delta)^t \diam(\X_0)$. Solving for the $t^*$ such that $\diam(\X_{t^*}) < \zeta$, we have that $\diam(\X_{t^*}) \le (1 - \delta)^{t^*} \diam(\X_{0}) < \zeta$. Since $t^*$ is an integer, the ceiling suffices.
\end{proof}

\begin{remark}
  \review{\cref{thm:geo_rate} shows one advantage of using a time-inhomogeneous process, since for a given $\delta$, we can find a schedule for $\epsilon_t$ such that $0< \delta\le \mathbf{K}_t(i,j) \le 1$, and thus controlling the rate of convergence. This is in contrast with the time-homogeneous process, where condensation would be defined using the same kernel, hence losing the benefit of an adaptive $\epsilon$. }
\end{remark}


For specific forms of kernels, this result can be translated to suggest concrete ways of setting the kernel parameters at each time step such that this bound holds, and diffusion condensation achieves well behaved linear convergence. 

\begin{proposition}\label{prop:kernel_specific}
For the following kernels the specific bandwidth update suffices for the result in \cref{thm:geo_rate} to hold.
\begin{enumerate}
    \item For the $\alpha$-decay kernel, $\exp(-\| x - y \|_2^\alpha / \epsilon^\alpha)$, $\epsilon_t \ge -\diam(\X_t)^\alpha / \log(\delta)$ suffices. For $\alpha = 2$, this defines the scheduling for the Gaussian kernel. For $\alpha = 1$, this defines a scheduling for the Laplace kernel.
    \item For the density normalized kernel $k_{\epsilon, \beta}(x,y)$ combined with the $\alpha$-decay, we define $\epsilon_t^\alpha \geq -\diam(\X_t)^\alpha/\log(N^{2\beta}\delta)$, and $\epsilon_t \geq -\diam(\X_t)^2/\log(N^{2\beta}\delta)$ for the Gaussian kernel. Then, \cref{thm:geo_rate} holds, for all $\delta\in(0, 1/N^{2\beta})$. The same holds if we replace $N$ with $q_{max,t}:=\max q(i)$.
\end{enumerate}
\end{proposition}
\begin{proof}
  To show (1) we need to define a scheduling of $\epsilon_t$ such that $\min_{x,y \in \mathsf{X}_t} k_\epsilon(x,y) \ge \delta > 0$. For the $\alpha$-decay kernel we have 
$\min_{x,y \in \mathsf{X}_t} k_\epsilon(x,y) = \exp(-(\diam(\X_t)/\epsilon_t)^\alpha) \geq \delta \implies \epsilon_t^\alpha \geq -\diam(\X_t)^\alpha/\log(\delta)$.
To show (2), we need $\epsilon_t$ such that $\min_{x,y}
k_{\epsilon,\beta}(x,y) \ge \delta$. We remark that $q(i)\leq q_{max,t}
\leq N$, hence $\min_{x,y} k_{\epsilon,\beta}(x,y)\geq \min_{x,y}
k_{\epsilon}(x,y)/N^{2\beta}$. Therefore, we find $\epsilon_t$ such that
$\min_{x,y} k_{\epsilon}(x,y)\geq N^{2\beta}\delta$ and conclude in the
same way as part (1).
\end{proof}
\begin{remark}
We note that to ensure a constant rate of convergence in the diameter, the kernel bandwidth $\epsilon_t$ in \cref{prop:kernel_specific} shrinks over time proportional to the square of the diameter. This is in contrast to previous work~\cite{brugnone_coarse_2019} where a doubling schedule was used. 
\end{remark}

\begin{remark}
The previous results also hold for $\mP_t^\tau$ with $\tau > 1$, as long as there exist $\delta>0$ such that all entries $\mP_t^\tau(i,j)>\delta/N$. Moreover, the rate of convergence with $\mP_t^\tau$ cannot be slower than with $\mP_t$, because we can always write $\X_{t+1} = \mP_t^\tau \X_t = \mP_t^{\tau-1} \mP_t \X_t$. Finally, for some $\mP_t$ that includes zero probabilities, it is possible for $\mP_t^\tau$ to be strictly pointwise positive, hence \cref{thm:diameter_conv} could be used for the process defined with $\mP_t^\tau$ instead of $\mP_t$.
\end{remark}




\section{Spectral properties of the condensation process}\label{sec:spectral}


We complement the geometric perspective of condensation from the
previous section with a spectral one, based on the idea of using
an orthonormal basis to express any function $f\colon\X \to \reals$
(abbreviated as $f\in\reals^N$) as a weighted sum of the eigenvectors of
$\mP_t$ for all $t$. This sum is then divided into two terms: a constant
term, and a nonconstant one. The former corresponds to the lowest
frequency of a function; it is constant for all $x\in \X_t$ and for all
$t$. The nonconstant term is the rest of the frequencies; it can vary
depending on the eigenvectors of $\mP_t$. We extend this
reasoning to the time-inhomogeneous diffusion $\inP{t}f$. Our main
result is \cref{thm: spec_thm}, which provides an upper bound on the
norm of the nonconstant term. For a specific choice of kernel and using
the coordinate function, this bound will converge to zero, hence in
\cref{cor: spec_convergence} we show how condensation converges to
a single point. Before presenting the main theorem,
\cref{sec: spectral_sub_simple_ex} introduces a simpler example to
give insight on the structure and the challenges of the proof.

\subsection{A simple condensation process}\label{sec: spectral_sub_simple_ex}
We consider a symmetric transition matrix $\mA_t$ based on $\X_t$, and the coordinate functions $f_i(x)$ which returns the i-th coordinate of $x$. In that case, $\mA_t$ is known as a \emph{bistochastic} matrix, and its stationary distribution is the uniform distribution. Since $\mA_t$ is symmetric, its eigenvectors $\{\phi_{t,i}\}_{i=1}^N$ form an orthonormal basis of $\reals^N$. Moreover, its ordered eigenvalues $\{\lambda_{t,i}\}_{i=1}^N$ are less than or equal to one, with $\lambda_{t,1}=1$. Because $\mA_t$ is row stochastic, and $\lambda_{t,1}=1$, we can define $\phi_{t,1} = N^{-1/2}\mathds{1}$ where $\mathds{1}$ is a vector of ones of size $N$. Given these properties, we can write any function $f\in\reals^N$ as 
\begin{equation*}
    f = \sum_{k=1}^N \langle f,\phi_{t,k} \rangle\phi_{t,k}.
\end{equation*}
By splitting this sum into two terms, we define the constant term $L_t(f):=\langle f,\phi_{t,1} \rangle\phi_{t,1} = (1/N) \langle f, \mathds{1} \rangle \mathds{1}$ and $H_t(f):= \sum_{k\geq 2} \langle f,\phi_{t,k} \rangle\phi_{t,k}$, hence $f = L_t(f)+H_t(f)$. After one condensation step, we get
\begin{equation*}
    \mA_0 f =\langle f,\phi_{0,1} \rangle \phi_{0,1} + \sum_{k=2}^N \lambda_{0,k} \langle f,\phi_{0,k} \rangle\phi_{0,k},
\end{equation*}
since $\lambda_{0,1} = 1$. Moreover, we note $L_0(f) = L_0(\mA_0 f)$, which is therefore invariant through the iterations of condensation. We note the resemblance with the graph Fourier transform, which uses the eigendecomposition of the Laplacian and treats the eigenvalues as frequencies, and eigenvectors as harmonics. Here, $L_t$ can be thought of as the lowest frequency term of the function. Whereas $H_t$ varies depending on the eigenvectors, it can be seen as the higher frequencies of the function. Since $L_t(f)$ is constant during condensation, showing the convergence of the process is equivalent to showing that $\norm{H_t(\mA^{(t)} f)}_2$ tends to zero as $t$ tends to infinity. Indeed, if this is true, by using the coordinate function $f_i$ we have
\begin{equation*}
    \lim_{t\to\infty}L_0(\mA^{(t)}f_i) = \langle f_i,N^{-1/2}\mathds{1} \rangle N^{-1/2}\mathds{1} = \Big[(1/N)\sum_{j=1}^N f_i(x_0(j))\Big]\mathds{1},
\end{equation*}
thus the process converges to the mean of the $N$ data points. 
What is left to show is that the norm of the term $H_t(\mA^{(t)}f)$ indeed converges to zero as $t$ goes to infinity. After the first condensation step, we have the following bound 
\begin{equation*}
    \norm{H_0(\mA_0 f)}_2^2 = \sum_{k=2}^N \lambda_{0,k}^2 |\langle f,\phi_{0,k} \rangle|^2 \leq  \lambda_{0,2}^2 \sum_{k=2}^N  |\langle f,\phi_{0,k} \rangle|^2 = \lambda_{0,2}^2 \norm{H_0(f)}_2^2\leq\lambda_{0,2}^2 \norm{f}_2^2,
\end{equation*}
and it can be deduced that $ \norm{H_t(\mA^{(t)} f)}_2^2 \leq \prod_{i=0}^t\lambda_{i,2}^2 \norm{f}_2^2$. Hence, by showing that $\prod_{i=0}^t\lambda_{i,2}^2$ tends to zero as $t$ tends to infinity, we could conclude that the process converges to a point. 

Our situation is more complex since many kernels are not symmetric, so
their eigenvectors do not form an orthonormal basis of $\reals^N$.
Here, we also benefited from the fact that the kernels
were bi-stochastic, hence they all had the same~(uniform) stationary
distributions. Generally, each kernel $\mP_t$ has
a different stationary distribution. This will be
reflected in the upper bound, as we have to consider the distance
between two consecutive stationary distributions.

\subsection{A general condensation process}\label{sec: spectral_sub_main}

In general, we consider a broader class of diffusion operators defined in \cref{sec:dc:notation} by $\mP_t = \mD_t^{-1}\mK_t$. \review{We recall that $\mK_t$ is symmetric with $0\leq\mK_t(i,j)\leq 1$, and the diagonal degree matrix is $\mD_t := \mathrm{diag}(d_t)$ where $d_t(i) := \sum_j \mK_t(i,j)$.} Moreover, the stationary distribution associated to $\mP_t$ is $\pi_t(i) = \|d_t\|_1^{-1}d_t(i)$, and $\mP_t$ is $d_t$-reversible, i.e. $d_t(i)\mP(i,j) = d_t(j)\mP(j,i)$.
Thus, its associated operator
\begin{equation}
\label{eq: spect_op}
    \mP_tf(x(i)) := \sum_{j=1}^N \mP_t(i,j)f(x(j))
\end{equation}
is self-adjoint with respect to the dot product $\langle f, g \rangle_{d_t} = \sum_x f(x(i))g(x(i))d_t(i)$. Denote $\{\psi_{t,i}\}_{i=1}^N$ as the eigenvectors of $\mP_t$ and  $\{\lambda_{t,i}\}_{i=1}^N$ as its eigenvalues arranged in decreasing order. Because $\mP_t$ is self-adjoint with respect to $\langle \cdot,\cdot\rangle_{d_t}$, its normalized eigenvectors are such that $\langle \psi_{t,i} , \psi_{t,j} \rangle_{d_t} = \delta_{ij}$, where $\delta_{ij} = 1$ if $i=j$ and $0$ otherwise. Therefore, we can write any function $f\in\reals^N$ as 
$ 
    f = \sum_{k=1}^N \langle f,\psi_{t,k} \rangle_{d_t}\psi_{t,k}.
$ 
Following the same steps as in \cref{sec: spectral_sub_simple_ex}, we want to find a constant term of the function. Since $\mP_t$ is row stochastic, and because $\lambda_{t,1} = 1$, we get $\psi_{t,1} = c \mathds{1}$ where $c$ is a constant. We can solve for $c$ using $\langle \psi_{t,1} , \psi_{t,1} \rangle_{d_t} = \langle c\mathds{1} , c\mathds{1} \rangle_{d_t} = 1$, which yields $c^2 = [\sum d_t(i)]^{-1} = ||d_t||_1^{-1}$. Hence, we define the constant term
\begin{equation*}
   L_t(f) :=  \langle f , \psi_{t,1} \rangle_{d_t}\psi_{t,1} = ||d_t||_1^{-1}\langle f , \mathds{1} \rangle_{d_t}\mathds{1} = \langle f , \mathds{1} \rangle_{\pi_t}\mathds{1},
\end{equation*}
and the nonconstant term as the rest of the sum, which varies depending on the eigenvectors, 
\begin{equation*}
    H_t(f) := \sum_{k=2}^N\langle f,\psi_{t,k} \rangle_{d_t}\psi_{t,k}.
\end{equation*}
We can write
    $\mP_t f =\langle f,\mathds{1} \rangle_{\pi_t} \mathds{1} + \sum_k \lambda_{t,k} \langle f,\psi_{t,k} \rangle_{d_t} \psi_{t,k}$,
by using the fact that $\mP_t$ is self-adjoint and $\lambda_{t,1}=1$. Most importantly, we note that the constant term of the function is not affected by condensation, i.e. $L_t(\mP_t f) = L_t(f)$. Consequently, to show that the condensation process converges to a single point, it is sufficient to show
\begin{equation}
\label{eq: H_to_zero}
    \lim_{t\to\infty}\norm{H_t(\mP^{(t)} f_i)}_2 = 0.
\end{equation}
Indeed, if \eqref{eq: H_to_zero} holds, then for a constant $C$, we get
   $\lim_{t\to\infty} \mP^{(t)}f_i(x) = C$
for all coordinates $i\in\{1,\dotsc,d\}$, and every $x\in\X_0$.
To achieve our goal, in \cref{thm: spec_thm} we find an upper bound on
$\norm{H_t(\mP^{(t)} f)}_2$, which we use in \cref{cor:
spec_convergence} to show convergence of the condensation process.
Before presenting these results, we introduce several lemmas, whose
proofs are provided in the supplementary materials.
\cref{lemma: bound Hpf_Hf} is the same as the upper bound we found in
\cref{sec: spectral_sub_simple_ex}, and will be beneficial when used
recursively. \cref{lemma:change_of_measure} is necessary since each
$\mP_t$ possibly has a different degree $d_t$, and it enables a change
of measure to $\norm{\cdot}_{d_s}$ from $\norm{\cdot}_{d_t}$. Lastly,
\cref{lemma: bound Ls-Lt} is beneficial when combined with the
observation that $H_t(f) = L_s(f) - L_t(f) + H_s(f)$.

\begin{lemma}
\label{lemma: bound Hpf_Hf}
For the operator $\mP_t$ and its second largest eigenvalue $\lambda_{t,2}$, we have the following bound on the norm
$ 
\norm{H_t(\mP_t f)}_{d_t}\leq\lambda_{t,2}\norm{H_t(f)}_{d_t},
$ 
for all functions $f\in \reals^N$. 
\end{lemma}

\begin{lemma}
\label{lemma:change_of_measure}
  For all functions $f\in \reals^N$, and two operators $\mP_t$ and $\mP_s$, the following inequalities hold
  $ 
      \norm{f}_{d_t}^2\leq\norm{d_t/d_s}_\infty\norm{f}^2_{d_s} \leq \left(\norm{d_t-d_s}_2+1\right)\norm{f}^2_{d_s}.
  $ 
\end{lemma}

\begin{lemma}
\label{lemma: bound Ls-Lt}
For all $f\in \reals^N$,
$ 
    \norm{L_t(\mP_t f)-L_s(\mP_t f)}_{d_t} \leq \lambda_{t,2}N^{1/2}\norm{d_s-d_t}_2\norm{H_t(f)}_{d_t}.
$ 
\end{lemma}

We are now ready to state and prove the main theorem of this section,
providing an upper bound on the norm of the nonconstant part
of a function. The upper bound mainly depends on two terms: the second
largest eigenvalue of each condensation operator, and the distance between their stationary distributions. The convergence
proof relies on this theorem.

\begin{theorem}
\label{thm: spec_thm}
For a condensation step $t$ and a collection of diffusion operators $\{\mP_k\}_{k=0}^t$, we have the following bound on the norm of the nonconstant term of the function $\inP{t}f$
\begin{equation*}
    \norm{H_t(\inP{t}f)}_{2} \leq \norm{d_0}_\infty^{1/2}\left[ \,\prod_{i=0}^{t-1}\lambda_{i,2} \right]\left[ \,\prod_{i=0}^{t-1}(1+N^{1/2}\norm{d_{i}-d_{i+1}}_2)^2 \right]\norm{f}_{2}.
\end{equation*}
\end{theorem}

\begin{proof}
We start by proving the following inequality
\begin{equation}
    \label{eq: proof_thm_spec_main_ineq}
    \norm{H_t(\inP{t-1}f)}_{d_t}\leq \left[ \,\prod_{i=0}^{t-1}\lambda_{i,2} \right]\left[ \,\prod_{i=0}^{t-1}(1+N^{1/2}\norm{d_{i}-d_{i+1}}_2)^2 \right]\norm{f}_{d_0}.
\end{equation}
We will prove \eqref{eq: proof_thm_spec_main_ineq} by induction. For $t = 1$, using \cref{lemma:change_of_measure} we obtain
\begin{equation}
    \norm{H_1(\mathbf{P}_0 f)}_{d_1} \leq (1+N^{1/2}\norm{d_{0}-d_{1}}_2) \norm{H_1(\mathbf{P}_0 f)}_{d_0}.\label{inproof: proof_th_first_induc}
\end{equation}
Moreover, using the fact that $H_1(f) = L_0(f)-L_1(f) + H_0(f)$, we write
\begin{align}
    \norm{H_1(\mathbf{P}_0 f)}_{d_0} &\leq\norm{L_0(\mathbf{P}_0 f)-L_1(\mathbf{P}_0 f)}_{d_0} + \norm{H_0(\mathbf{P}_0 f)}_{d_0}\notag\\
    &\leq \lambda_{0,2}N^{1/2}\norm{d_{0}-d_{1}}_2\norm{H_0(f)}_{d_0} + \lambda_{0,2}\norm{H_0(f)}_{d_0}\label{inproof: th_Hpf_to_hf}\\
    &\leq \lambda_{0,2}(1+N^{1/2}\norm{d_{0}-d_{1}}_2)\norm{f}_{d_0}\label{inproof: proof_th_second},
\end{align}
where the inequality \eqref{inproof: th_Hpf_to_hf} is obtained by \cref{lemma: bound Ls-Lt} and \cref{lemma: bound Hpf_Hf}. Combining the equations \eqref{inproof: proof_th_first_induc} and \eqref{inproof: proof_th_second} yields $\norm{H_1(\mathbf{P}_0 f)}_{d_0}\leq \lambda_{0,2}(1+N^{1/2}\norm{d_{0}-d_{1}}_2)^2\norm{f}_{d_0}$. We have shown that \eqref{eq: proof_thm_spec_main_ineq} is true for $t=1$. Now assume it is true up to $t-1$, we want to show it implies that it is true for $t$. The proof is similar to the base case. From \cref{lemma:change_of_measure}, we obtain
\begin{equation}
    \norm{H_t(\inP{t-1} f)}_{d_t}
    \leq (1+N^{1/2}\norm{d_{t-1}-d_{t}}_2) \norm{H_t(\inP{t-1} f)}_{d_{t-1}} \label{inproof: proof_th_3}.
\end{equation}
Moreover, following the same steps as \eqref{inproof: proof_th_second} with \cref{lemma: bound Ls-Lt} and \cref{lemma: bound Hpf_Hf}, we obtain
$ 
    \norm{H_t(\inP{t-1} f)}_{d_{t-1}} \leq \lambda_{{t-1},2}(1+N^{1/2}\norm{d_{t-1}-d_{t}}_2)\norm{H_{t-1}(\inP{t-2} f)}_{d_{t-1}}.
$ 
Combining this last inequality with \eqref{inproof: proof_th_3} yields
$ 
   \norm{H_t(\inP{t-1} f)}_{d_{t}}\leq \lambda_{t-1,2}(1+N^{1/2}\norm{d_{t-1}-d_{t}}_2)^2 \norm{H_{t-1}(\inP{t-2} f)}_{d_{t-1}},
$ 
and we prove \eqref{eq: proof_thm_spec_main_ineq} by applying the inductive hypothesis. We can now conclude the proof by
\begin{align}
    \norm{H_t(\inP{t} f)}_{2}&\leq \norm{1/d_t}_\infty^{1/2}\norm{H_t(\inP{t} f)}_{d_t}
    \leq \norm{H_t(\inP{t} f)}_{d_t}\label{eq:proof_spec_conclude_1}\\
    &\leq \lambda_{t,2}\norm{H_t(\inP{t-1}f)}_{d_t}
    \leq \norm{H_t(\inP{t-1}f)}_{d_t}\label{eq:proof_spec_conclude_2}\\
    & \leq \left[ \,\prod_{i=0}^{t-1}\lambda_{i,2} \right]\left[ \,\prod_{i=0}^{t-1}(1+N^{1/2}\norm{d_{i}-d_{i+1}}_2)^2 \right]\norm{f}_{d_0}\label{eq:proof_spec_conclude_3}\\
    &\leq \norm{d_0}_\infty^{1/2}\left[ \,\prod_{i=0}^{t-1}\lambda_{i,2} \right]\left[ \,\prod_{i=0}^{t-1}(1+N^{1/2}\norm{d_{i}-d_{i+1}}_2)^2 \right]\norm{f}_{2}\notag,
\end{align}
where the inequalities \eqref{eq:proof_spec_conclude_1} and \eqref{eq:proof_spec_conclude_2} are obtained by \cref{lemma:change_of_measure} and \cref{lemma: bound Hpf_Hf}, the inequality \eqref{eq:proof_spec_conclude_3} is justified by \eqref{eq: proof_thm_spec_main_ineq} which we have just shown. Finally, the last inequality is due to \cref{lemma:change_of_measure}.
\end{proof}

\begin{remark}
  \review{\Cref{thm: spec_thm} is valid for a general collection of diffusion operators constructed from a kernel like the ones presented in \cref{sec:dc:parameters}. In particular, it includes the collection of operators created by the diffusion condensation algorithm and the time-homogeneous process. For the latter, the product $ \prod(1+N^{1/2}\norm{d_{i}-d_{i+1}}_2)^2$ is equal to one, thus the rate of convergence only depends on the second largest eigenvalue of the diffusion operator. Allowing for time-inhomogeneity enables controlling the eigenvalue \emph{during} the process, for example, by defining an adaptive bandwidth parameter, but comes at the cost of having to consider the rate of change of the degrees. }
\end{remark}

We recall our initial argument that to show the convergence of the condensation process it is sufficient to use the coordinate function $f_i$ and to show that the norm of the nonconstant term $\norm{H_t(\inP{t} f_i)}_{2}$ converges to zero. This is achieved in the next corollary, \review{for which we require the following assumption on successive degree functions
\begin{equation}\label{eq: assumption_degree}
    \sum_{k=0}^\infty \|d_k - d_{k+1}\|_2<\infty.
\end{equation}}

\begin{corollary}
\label{cor: spec_convergence}
For a family of diffusion operators $\{\mP_t\}_{t\in\naturals}$ defined by \eqref{eq: spect_op} such that their second largest eigenvalues are all less or equal to $1-\delta$, where $\delta\in(0,1)$, and that $(d_t)_{t\in\naturals}$ \review{respects \eqref{eq: assumption_degree}}, then the condensation process converges to a (single) point as $t$ tends to infinity. 
\end{corollary}
\begin{proof}
Using the coordinate function $f_i$ and the upper bound from \cref{thm: spec_thm}, we have 
\begin{equation}
\label{eq: spec_cor_main_eq}
    \norm{H_t(\inP{t} f_i)}_{2}\leq \norm{d_0}_\infty^{1/2}\left[ \,1-\delta \right]^t\left[ \,\prod_{k=0}^{t-1}(1+N^{1/2}\norm{d_{k}-d_{k+1}}_2)^2 \right]\norm{f_i}_{2},
\end{equation}
since $\lambda_{t,2}\le 1-\delta$ for all $t$. Note that the quantities $\norm{d_0}_\infty^{1/2}$ and $\norm{f_i}_{2}$ are both finite. Furthermore, by assumption, the sequence $(d_t)_{t\in\naturals}$ \review{satisfies \eqref{eq: assumption_degree}, thus}
\begin{equation*}
    \lim_{t\to \infty} \prod_{k=0}^{t-1}(1+N^{1/2}\norm{d_{k}-d_{k+1}}_2)^2 < \infty.
\end{equation*}
The upper bound converges to zero since $\lim_{t\to\infty}\left[
\,1-\delta \right]^t = 0$, and because $\lim_{t\to \infty}\inP{t} f_i
=  \langle f_i , \mathds{1} \rangle_{\pi}\mathds{1}$, we conclude that
all points have the same i-th coordinate for all $i\in\{1,\dotsc,d\}$.
\end{proof}

We conclude this section by identifying kernels for which we can find analytic conditions that respect the assumptions of the previous corollary, hence producing a condensation process that converges to a single point. \review{First we introduce the following lemma regarding the degrees assumption \eqref{eq: assumption_degree}.
\begin{lemma}\label{lemma: degrees_monotone}
  If $\lim_{k\to\infty}d_k$ exists, and the degrees are such that $d_{k}(i)\leq d_{k+1}(i)$ except for a finite number of condensation steps, then assumption \eqref{eq: assumption_degree} is verified.
\end{lemma}
\begin{proof}
  We note $d_\infty:=\lim_{k\to\infty}d_k$, and we recall $1\leq d_k(i)\leq N$. Without loss of generality, we assume that all degrees after $\ell$ condensation steps respect the monotonic assumption, since $\|\cdot\|_2\leq\|\cdot\|_1$, we will show $\sum_{k=\ell}^\infty\|d_k - d_{k+1}\|_1<\infty$ to complete the proof. We have
\begin{align*}
    \sum_{k=\ell}^\infty\sum_{i=1}^N |d_k(i) - d_{k+1}(i)| = \sum_{i=1}^N \sum_{k=\ell}^\infty d_{k+1}(i)-d_k(i) = \sum_{i=1}^N d_\infty(i) - d_\ell(i) \leq \sum_{i=1}^N (N-1) \leq \infty,
\end{align*}
where the first equality comes from the increasing degrees assumption and interchanging the order of summation, the second equality is due to the telescoping sum.
\end{proof}
In the following, we assume the conditions of~\cref{lemma: degrees_monotone} to be verified. This is consistent with our experiments, as, after a few condensation steps, we observe that all pairwise distances decrease, hence each dimension of the degrees is increasing. For the assumption on $\lambda_{i,2}$ we analyze the diffusion operator.} 
Since $\mP_t$ is reversible with respect to $\pi_t$, we can use Prop. 1 of Diaconis and Stroock \cite{diaconis1991geometric} to find an upper bound on the second largest eigenvalue. They show that $ \lambda_{t,2} \leq 1-1/\kappa_t$, where
\begin{equation*}
    \kappa_t := \max_{i,j} \frac{\pi_t(i)\pi_t(j)}{\pi_t(i)\mP_t(i,j)} \leq \max_{i,j} \frac{d_t(i)}{\mK_t(i,j)} \leq \frac{d_{max,t}}{\min_{i,j}\mK_{t}(i,j)},
\end{equation*}
and $d_{max,t} := \max_i d_t(i)$. To respect the assumptions of \cref{cor: spec_convergence}, we define $\epsilon_t$ such that 
\begin{equation*}
     \lambda_{t,2} \leq 1 - \frac{\min_{x,y\in\X_t}k_{\epsilon}(x,y)}{d_{max,t}} \leq 1-\delta.
\end{equation*}
Thus, for the $\alpha$-decay kernel (\cref{def: alpha-decay Kernel}), we must define a schedule of the bandwidth parameter $\epsilon_t$, such that $\epsilon_t^\alpha \geq -\diam(\X_t)^\alpha / (-\log(\delta d_{max,t}))$, which extends schedules for the Gaussian and Laplace kernel. For these kernels, we always need $\delta\in(0,1/d_{max,t})$, but we note that $d_{max,t} \leq N$, thus avoiding the case where $\delta$ tends to $0$ as $t$ tends to infinity. A similar result is obtained for the density normalized kernel $k_{\epsilon,\beta}$ (\cref{def: Anisotropic Kernel}), since
\begin{equation*}
    \max_{x,y\in\X_t} \sum_y k_{\epsilon,\beta}(x,y) \leq N,\,\,\text{and } \min_{x,y\in\X_t} k_{\epsilon,\beta}(x,y) \geq \frac{\min_{x,y\in\X_t}k(x,y)}{q_{max,t}^{2\beta}}.
\end{equation*}
Combining these two bounds yields the following requirement for the density normalized kernel
\begin{equation*}
     \lambda_{t,2} \leq 1 - \frac{\min_{x,y}k_{\epsilon}(x,y)}{Nq_{max,t}^{2\beta}} \leq 1-\delta.
\end{equation*}
We can find similar schedule for each of the previous kernels, since $\min_{x,y}k_t(x,y)$ can be lower bounded by a function of the diameter. For instance, we find $\epsilon_t \geq -\diam(\X_t)^2/\log(\delta Nq^{2\beta}_{max,t})$ for the anisotropic Gaussian kernel. This adaptative parametrization of the bandwidth parameter guarantees that the condensation process will converge to a point for these kernels.
%

\begin{remark}\label{rem: spec_tau}
  These results can be generalized to $\mP_t^\tau$, for any
  $\tau\in\naturals$. We can write $\mP_t^\tau f = L_t(f) + \sum_k
  \lambda^\tau_{t,k} \langle f,\psi_{t,k} \rangle_{d_t} \psi_{t,k}$,
  hence $\norm{H_t(\mP^\tau_t f)}_{d_t}\le\norm{H_t(\mP_t
  f)}_{d_t}\le\lambda_{t,2}\norm{H_t(f)}_{d_t}$, since
  $|\lambda_{t,2}|\le 1$. Thus, \cref{thm: spec_thm} can be used to
  prove convergence of the process.
\end{remark}

\begin{remark}
  Both \cref{thm: spec_thm} and \cref{cor:
  spec_convergence} are valid for a broad class of diffusion operators, in
  particular those with finite support or a wider family of random walks
  on a graph. This differs from the geometric \cref{thm:diameter_conv},
  which is restricted to strictly-positive kernels.
  It is also possible to leverage information from the
  underlying structure of the data to characterize the convergence of
  the condensation process. For example, for a random walk on
  a graph, the second-largest eigenvalue is influenced by the
  connectivity of the graph; a highly-connected graph would yield
  a small eigenvalue, hence converging faster. For the Box kernel,
  assuming monotone convergence of degrees, \cref{cor: spec_convergence} can
  be used to analyze overall convergence by evaluating the
  second largest eigenvalue at different condensation times. 
\end{remark}

\begin{remark}
  The degree convergence assumption we use \eqref{eq: assumption_degree} assumes that the
  process converges to a stable representation $\X_M$, without any
  assumption on $\X_M$. In practice, since transition operators are
  contractive, we observe that this assumption is easily respected (from \cref{lemma: degrees_monotone}).  \review{It is worth noting that \eqref{eq: assumption_degree} can be controlled for random walks on k-nearest-neighbor graphs (since $d_t(i) = k$).} \cref{cor: spec_convergence}, bounding the second largest eigenvalue,
  then guarantees that $\X_M$ is a single point.
\end{remark}

\section{Topological properties of the condensation process}\label{sec:topological}

Having previously proved convergence properties, we now take on a coarser perspective
and characterize topological, i.e., \emph{structural}, properties of the diffusion
condensation process. To this end, we note that the multiresolution structure provided
by diffusion condensation
naturally relates to recent advances in using computational topology to
understand the ``shape'' of data geometry at varying scales. To
elucidate this connection, \cref{sec:int-homology} and \cref{sec:amb-homology}
introduce two perspectives for integrating topological information into
the data geometry uncovered by the diffusion condensation process, i.e.,
\begin{inparaenum}[(i)]
\item \review{\emph{condensation homology}} for describing the topology of the diffusion condensation process itself, and
  \item \review{\emph{persistent homology} based on Vietoris--Rips
    complexes} for describing each step of the diffusion
  condensation process, \review{thus closing the loop to the
  previously-provided geometric notions}.
\end{inparaenum}
We provide a brief review of relevant topological data analysis (TDA)
notions in \cref{sec:summary_homology} and in the supplementary material.
\autoref{fig:Double annulus intrinsic diffusion homology} and \autoref{fig:Double annulus ambient diffusion homology}
depict the two types of topological descriptions.
\review{
  Readers familiar with topological data analysis may recognize that our
  two perspectives may also be seen as slices of
  a special bifiltration, i.e., a filtration with two parameters.
  However, since bifiltrations are known to be computationally more
  challenging~\cite{Lesnick15}, we defer their treatment to future work.
}

\begin{figure}[tbp]
  \centering
  \pgfplotsset{%
    /pgfplots/group/every plot/.append style = {%
      mark size = 0.5pt,
    },
  }
  \begin{tikzpicture}
    \begin{groupplot}[%
      group style = {%
        group size     = 3 by 3,
        xlabels at     = edge bottom,
        xticklabels at = edge bottom,
        ylabels at     = edge left,
        yticklabels at = edge left,
        horizontal sep = 0.5cm,
        vertical sep   = 0cm,
      },
      width  = 4cm,
      height = 4cm,
      ticks  = none,
    ]
      \nextgroupplot[%
        xmin = -0.30,
        xmax =  0.90,
        ymin = -0.10,
        ymax =  1.10,
      ]
        \addplot[only marks] table {./Data/Double_Annulus/double_annulus_gaussian_n128_t000.txt};
      \nextgroupplot
        \addplot[only marks] table {./Data/Hyperuniform_Circle/hyperuniform_circle_gaussian_n128_t00.txt};
      \nextgroupplot
        \addplot[only marks] table {./Data/Petals/petals_gaussian_n128_t00.txt};

      \nextgroupplot[%
        ytick     = \empty,
        ticks     = both,
        xtick pos = left,
      ]
      \addplot[no marks] table {./Data/Double_Annulus/double_annulus_gaussian_n128_intrinsic_diffusion_homology.txt};

      \nextgroupplot[%
        ticks     = both,
        xtick pos = left,
        ytick     = \empty,
      ]
        \addplot[no marks] table {./Data/Hyperuniform_Circle/hyperuniform_circle_gaussian_n128_intrinsic_diffusion_homology.txt};

      \nextgroupplot[%
        ticks     = both,
        xtick pos = left,
        ytick     = \empty,
      ]
        \addplot[no marks] table {./Data/Petals/petals_gaussian_n128_intrinsic_diffusion_homology.txt};

      \nextgroupplot[%
        xlabel    = {$t$},
        ticks     = both,
        xtick pos = left,
        ytick pos = bottom,
      ]
        \addplot[no marks] table {./Data/Double_Annulus/double_annulus_gaussian_n128_total_persistence.txt};

      \nextgroupplot[%
        xlabel    = {$t$},
        ticks     = both,
        xtick pos = left,
        ytick pos = bottom,
      ]
        \addplot[no marks] table {./Data/Hyperuniform_Circle/hyperuniform_circle_gaussian_n128_total_persistence.txt};

      \nextgroupplot[%
        xlabel    = {$t$},
        ticks     = both,
        xtick pos = left,
        ytick pos = bottom,
      ]
        \addplot[no marks] table {./Data/Petals/petals_gaussian_n128_total_persistence.txt};
    \end{groupplot}
  \end{tikzpicture}

  \caption{%
    An illustration of \review{\emph{condensation homology}} for the ``double
    annulus'' dataset, the ``hyperuniform circle'' dataset \review{where $n$ points are evenly spaced around the circle with $\frac{2 \pi}{n}$ radians between them}, and the
    ``petals'' dataset. The upper row depicts the original dataset at $t
    = 0$; the middle row depicts the \review{condensation homology}
    barcode, i.e., a summary of topological activity over all
    condensation iterations; the lower row depicts topological activity
    curves~(cumulative sums of lengths in the \review{condensation homology barcode}).
  }
  \label{fig:Double annulus intrinsic diffusion homology}
\end{figure}
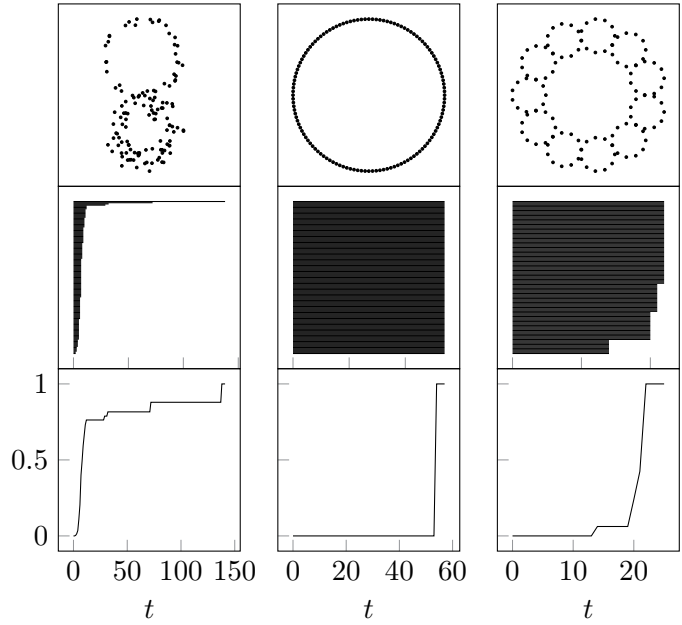

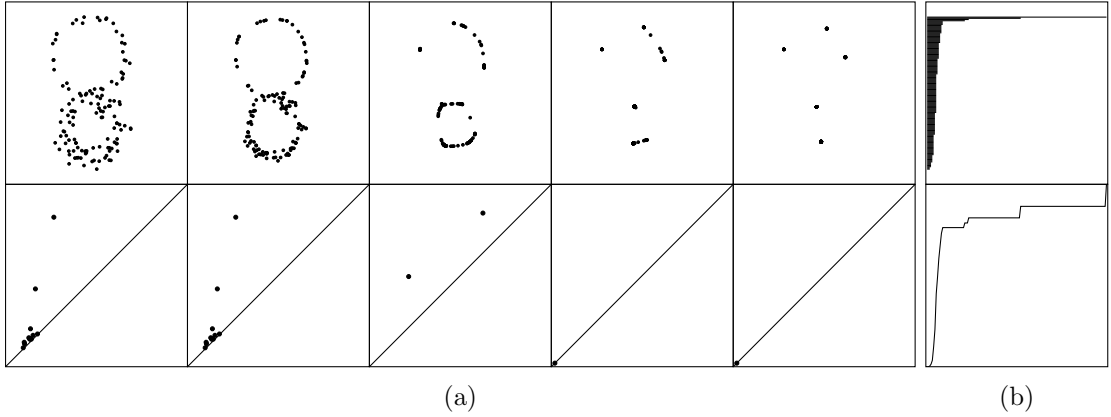
\begin{figure}[tbp]
  \centering
  \subcaptionbox{}{%
    \begin{tikzpicture}
      \begin{groupplot}[%
        group style = {%
          group size     = 5 by 2,
          ylabels at     = edge left,
          horizontal sep = 0cm,
          vertical sep   = 0cm,
        },
        unit vector ratio* = 1 1 1,
        %
        xmin = -0.30,
        xmax =  0.90,
        ymin = -0.10,
        ymax =  1.10,
        width  = 4cm,
        height = 4cm,
        ticks  = none,
      ]

      \pgfplotsset{%
        /pgfplots/group/every plot/.append style = {%
          mark size = 0.5pt,
        },
      }
        \nextgroupplot
        \addplot[only marks] table {./Data/Double_Annulus/double_annulus_gaussian_n128_t000.txt};

        \nextgroupplot
        \addplot[only marks] table {./Data/Double_Annulus/double_annulus_gaussian_n128_t001.txt};

        \nextgroupplot
        \addplot[only marks] table {./Data/Double_Annulus/double_annulus_gaussian_n128_t005.txt};

        \nextgroupplot
        \addplot[only marks] table {./Data/Double_Annulus/double_annulus_gaussian_n128_t008.txt};

        \nextgroupplot
        \addplot[only marks] table {./Data/Double_Annulus/double_annulus_gaussian_n128_t010.txt};

        \nextgroupplot[
            xmin      = -0.01,
            xmax      =  0.50,
            ymin      = -0.01,
            ymax      =  0.50,
            mark size = 0.75pt,
        ]

        \addplot[only marks] table {./Data/Double_Annulus/double_annulus_gaussian_n128_persistence_diagram_t000.txt};
        \addplot[domain = {-0.05:0.50}] {x};

        \nextgroupplot[
            xmin      = -0.01,
            xmax      =  0.50,
            ymin      = -0.01,
            ymax      =  0.50,
            mark size = 0.75pt,
        ]

        \addplot[only marks] table {./Data/Double_Annulus/double_annulus_gaussian_n128_persistence_diagram_t001.txt};
        \addplot[domain = {-0.05:0.50}] {x};

        \nextgroupplot[
            xmin      = -0.01,
            xmax      =  0.50,
            ymin      = -0.01,
            ymax      =  0.50,
            mark size = 0.75pt,
        ]

        \addplot[only marks] table {./Data/Double_Annulus/double_annulus_gaussian_n128_persistence_diagram_t005.txt};
        \addplot[domain = {-0.05:0.50}] {x};

       \nextgroupplot[
            xmin      = -0.01,
            xmax      =  0.50,
            ymin      = -0.01,
            ymax      =  0.50,
            mark size = 0.75pt,
        ]

        \addplot[only marks] table {./Data/Double_Annulus/double_annulus_gaussian_n128_persistence_diagram_t008.txt};
        \addplot[domain = {-0.05:0.50}] {x};

        \nextgroupplot[
            xmin      = -0.01,
            xmax      =  0.50,
            ymin      = -0.01,
            ymax      =  0.50,
            mark size = 0.75pt,
        ]

        \addplot[only marks] table {./Data/Double_Annulus/double_annulus_gaussian_n128_persistence_diagram_t010.txt};
        \addplot[domain = {-0.05:0.50}] {x};

      \end{groupplot}
    \end{tikzpicture}
  }%
  \subcaptionbox{}{%
    \begin{tikzpicture}
      \begin{groupplot}[%
        group style = {%
          group size     = 1 by 2,
          ylabels at     = edge left,
          horizontal sep = 0cm,
          vertical sep   = 0cm,
        },
        width              = 4cm,
        height             = 4cm,
        ticks              = none,
      ]
        \nextgroupplot[
          width  = 4cm,
          height = 4cm,
          xmin   =  -1.0,
          xmax   = 139.0,
        ]
          \addplot[no marks] table {./Data/Double_Annulus/double_annulus_gaussian_n128_intrinsic_diffusion_homology.txt};

        \nextgroupplot[
          xmin   =   -1.0,
          xmax   =  139.0,
          ymin   =    0.0,
          ymax   =    1.0,
        ]
          \addplot[no marks] table {./Data/Double_Annulus/double_annulus_gaussian_n128_total_persistence.txt};
      \end{groupplot}
    \end{tikzpicture}
  }%
  \caption{%
    Left: \review{\emph{Persistent homology}} for the ``double
    annulus'' dataset. Following different steps in the diffusion
    condensation process~(upper row), we obtain a sequence of
    persistence diagrams~(lower row) that summarize the one-dimensional
    topological features, i.e., the cycles, in the dataset. Right: The
    \review{\emph{condensation homology}}~(top) and the topological
    activity curve~(bottom) of the dataset for comparison
    purposes.
  }%
  \label{fig:Double annulus ambient diffusion homology}
\end{figure}

\subsection{A brief summary of persistent homology}\label{sec:summary_homology}
\review{
Persistent homology~\cite{Barannikov94,Edelsbrunner02} is a method
from the field of computational topology, which develops tools for
obtaining and analyzing topological features of datasets. Given its
beneficial robustness properties~\cite{Cohen-Steiner07}, persistent
homology has received a large degree of attention from the
machine learning community~\cite{Hensel21}.
}%

\review{%
We first introduce the underlying concept of simplicial homology.
For a simplicial complex $\simplicialcomplex$, i.e.\
a generalized graph with higher-order connectivity information in the
form of cliques, simplicial homology employs matrix reduction algorithms to
assign $\simplicialcomplex$ a family of groups, the
\emph{homology groups}.
The $d$th homology group
$\homologygroup{d}\left(\simplicialcomplex\right)$ of
$\simplicialcomplex$ contains equivalence classes of
\mbox{$d$-dimensional} topological features, such as connected
components~($d = 0$), cycles/tunnels~($d = 1$), and voids~($d = 2$).
These features are also known as homology classes. 
Homology groups are typically summarized by their ranks, thereby
obtaining a simple invariant ``signature'' of a manifold.
For instance, a circle in $\reals^2$ has one feature with $d = 1$, i.e.,
a cycle, and one feature with $d = 0$, i.e., a connected component.
In practice, we are dealing with a point cloud $\X$ and a metric, such
as the Euclidean distance. In this setting, \emph{persistent homology}
now creates a sequence of nested simplicial complexes, making it
possible to track the changes in homology groups---and thus the changes
in topology---over multiple scales~(with the understanding that
real-world data sets necessitate such a multi-scale perspective,
a single scale being too restrictive).
This is achieved by constructing a special simplicial complex, the
Vietoris--Rips complex~\cite{Vietoris27}.
For $0 \leq \epsilon < \infty$, the Vietoris--Rips complex of
$\X$ at scale $\epsilon$, denoted by
$\vietoris_{\epsilon}\left(\X\right)$, contains all
simplices~(i.e., subsets) of $\X$ whose elements $\{x_0, x_1,
\dots\}$ satisfy $\dist\left(x_i,
x_j\right) \leq \epsilon$ for all $i$, $j$.
Calculating topological features of $\vietoris_{\epsilon}$ results in
a set of tuples of the form $(\epsilon_i, \epsilon_j, d)$, where
$\epsilon_i \in \reals$ refers to a threshold at which a topological
feature was ``created'', i.e., the threshold at which it occurred for
the first time in $\vietoris_\epsilon$. Likewise, $\epsilon_j \in
\reals$ refers to the threshold at which the feature was destroyed.
Last, $d$ indicates the dimension of the respective feature. Together,
the features of dimension~$d$ form the $d$-dimensional \emph{persistence
diagram}, a topological descriptor containing the point~$(\epsilon_i,
\epsilon_j)$ for every such tuple above.
For example, when~$d= 0$, the threshold $\epsilon_j$ denotes at which
distance two connected components in a dataset are merged into one.
}

\subsection{Condensation homology}\label{sec:int-homology}

The formulation of the diffusion condensation process, with its merge
step for close points,  induces changes in the topological structure of
the datasets. This will result in \emph{one} topological descriptor
summarizing them.
We first define a filtration, \review{i.e., an ordering of subsets of
the data, such that we obtain a sequence of nested simplicial complexes},
that is intrinsic to the diffusion
condensation process, being compatible with the algorithm in \cref{sec:dc:algorithm}.
\review{The filtration is based on the idea of first extracting subsets
  of the data that satisfy a pairwise distance requirement---similar to
  the Vietoris--Rips filtration, which we shall describe in
  \cref{sec:amb-homology}---and assign them a weight based on the
condensation time~$t$. This weight is used to track topological changes
during the condensation process.}
\begin{definition}[Condensation homology filtration]
  Given a merge threshold $\zeta\in\reals_{> 0}$, we define the \emph{intrinsic
  condensation filtration} for $t \in\naturals$ \review{as the filtration
  arising from the sequence of simplicial complexes}
  \begin{equation}
  \vietoris_t(\X, \zeta) := \left\{ \sigma \subseteq \Xt \mid
  \dist\left(x_t\left(i\right), x_t\left(j\right)\right) \leq \zeta \text{
  for all $x(i), x(j) \in \sigma$} \right\} \bigcup_{t' = 0}^{t - 1}
  \vietoris_{t'}(\X, \zeta),
    \label{eq:Condensation homology filtration}
  \end{equation}
  with $\vietoris_0(\X, \zeta) := \left\{ \sigma \subseteq \X \mid
  \dist\left(x\left(i\right), x\left(j\right)\right) \leq \zeta
  \right\}$. \review{The weight function $\mathrm{w}\colon 2^{\X} \to
  \naturals$ for each $\vietoris_t(\X, \zeta)$ is defined by setting}
  $\mathrm{w}(\{i\}) := 0$ for a \mbox{0-simplex} $\{i\}$, and by
  \review{setting}
  $\mathrm{w}(\{i, j\}) := \min\left\{t \mid \left\{i, j\right\} \in
  \vietoris_t(\X, \zeta) \right\}$
  for each \mbox{1-simplex} $\{i, j\}$, i.e., we use the first~$t$ such
  that the two points are in a $\zeta$-neighborhood. The weight function
  can be extended to higher-dimensional simplices inductively by taking the
  maximum.
\end{definition}

%
\begin{lemma}
  Using \cref{eq:Condensation homology filtration} results in a nested sequence
  of simplicial complexes. \review{We thus obtain a valid filtration from which
  we may calculate topological features.}
\end{lemma}
\begin{proof}
  The nesting property is achieved by taking the union in
  \eqref{eq:Condensation homology filtration}. Hence, $\vietoris_t(\X, \zeta)$
  can only grow, which ensures that consecutive complexes are nested. The
  weights are not guaranteed to be unique, but we obtain
  a consistent ordering by using the indices of the respective points.
\end{proof}

Intuitively, the condensation homology filtration measures at which
iteration step~$t$ two points move into their $\zeta$-neighborhood for
the first time.
\review{There are two differences to a traditional Vietoris--Rips filtration as
used in \cref{sec:amb-homology}}.
First, we enforce the nesting condition of a filtration by taking the
union of all simplicial complexes for previous time steps; this is
necessary because, depending on the threshold $\zeta$, we cannot
guarantee that points remain within a $\zeta$-neighborhood.\footnote{%
  \review{For readers familiar with computational topology, we want to
    remark that using zigzag persistent homology~\cite{Carlsson09a},
    which does not require stringent nesting conditions for filtrations,
  would also be a possibility. We will consider such a perspective in
future work.}
}
The second difference is that we filter over diffusion condensation
iterations instead of distance thresholds\review{, necessitating the use
  of an additional weight function~(as opposed to using the distances
between points)}. Since diffusion condensation
results in changes of local distances, this filtration captures the
intrinsic behavior of the process. For now, we only add
\mbox{1-simplices} and \mbox{0-simplices} to every $\vietoris_t(\X, \zeta)$,
but the definition generalizes to higher-order simplices.
\review{We define \emph{condensation homology} to be the degree-0 persistent
homology of $\vietoris_t(\X, \zeta)$ under the weight function defined
above.}

Intuitively, we initially treat each data point~$x_i$ as a \mbox{0-simplex},
creating its own homology class and identify homology classes over
different time steps~$t$, i.e., the homology class of $x_t(i)$ and
$x_{t'}(i)$ for $t \neq t'$ is considered to be the same.
As the geometry of the underlying point cloud changes during
each iteration, points start to progressively cluster.  Whenever
a \emph{merge} event happens~(see line~\ref{lst:Merge event} in the
diffusion condensation algorithm), we let the homology class
corresponding to the vertex with the lower index continue, while we
\emph{destroy} the other homology class. Given our weight function, such
an event results in a tuple of the form $(0, t)$, where~$t$ denotes the diffusion
condensation iteration. This can also be considered as a persistence diagram
arising from a distance-based filtration of an abstract input dataset~(hence,
every tuple contains a~$0$; all homology classes---i.e., all
vertices---are present at the start of the diffusion condensation
process). Our weight function can be interpreted as a ``temporal
distance;'' the
distance between pairs of vertices~$(i, j)$ is given by calculating the
value of~$t$ for which their spatial distance falls below the merge
threshold~$\zeta$ for the first time, i.e., $\dist(i, j) := \min\{ t \mid
\dist\left(x_t(i), x_t(j)\right) \leq \zeta\}$.
A convenient representation can be obtained using
a \emph{persistence barcode}~\cite{Ghrist08a}, i.e., a representation in which the
lifespan of each homology class is depicted using a bar. Longer bars
indicate more prominent clusters or groupings in data~\cite{kuchroo_topological_2021}.
\autoref{fig:Double annulus intrinsic diffusion homology} illustrates this for
a ``double annulus'' dataset, which does not give rise to a complex set
of clusters, as indicated by the existence of few long bars in such
a barcode.
%

\begin{figure}[tbp]
  \centering
  \subcaptionbox{Gaussian kernel}{%
    \input{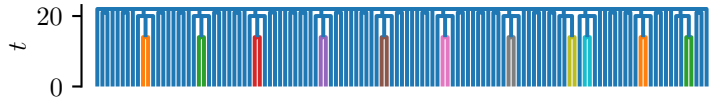}
  }\\
  \subcaptionbox{$\alpha$-decay kernel}{%
    \input{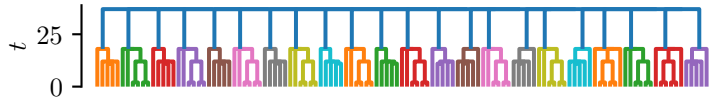}
  }
  \caption{%
    Two dendrograms obtained on the ``petals'' dataset. The different
    condensation behavior exhibited by different kernels~(see
    \autoref{fig:Example kernels}) also manifests itself in the
    dendrograms.
  }
  \label{fig:Dendrograms petals}
\end{figure}

Notice that the persistence pairing~$\mathcal{P}$ corresponding to the
condensation homology carries all the information about the
hierarchy of merges obtained during the diffusion condensation process. 
Specifically, $\mathcal{P}$ consists of pairs of the form~$(\{u\},
\{v, w\})$, where $\{u\}$ is a vertex and $\{v, w\}$
is an edge between two vertices. We can use these edges to construct
a tree of merges, i.e., a \emph{dendrogram}~(see \autoref{fig:Dendrograms petals}).
This perspective will be useful later on when we show how
diffusion condensation generalizes existing hierarchical clustering
methods.

\subsection{Persistent homology of the diffusion process}\label{sec:amb-homology}

\review{As a more expressive---but also more complicated---description
of topological features in the condensation process, we calculate
persistent homology of the input dataset~$\X$ at every condensation
iteration.} \review{To this end,} we calculate a Vietoris--Rips
complex for each point cloud $\X_t$ of the diffusion
condensation process, denoting the Vietoris--Rips complex
of~$\mathsf{X}$ at diffusion time~$t$ as
\reviewtwo{$\vietoris_{\zeta}(\mathsf{X}, t) := \left\{ \sigma \subseteq
  \Xt \mid \dist\left(x_t\left(i\right), x_t\left(j\right)\right) \leq
\zeta \text{ for all $x(i), x(j) \in \sigma$} \right\}$}~\review{(this notation was
chosen to contrast with $\vietoris_t(\mathsf{X}, \zeta)$ from
\eqref{eq:Condensation homology filtration}, in which~$t$ is varied as
the filtration parameter while $\zeta$ is kept fixed)}.
In the following, we will prove that the
topological features of \review{$\vietoris_{\zeta}(\mathsf{X}, t)$}
converge as the diffusion condensation process converges.
\review{To this end, we make use of the \emph{bottleneck
  distance}~$\bottleneck(\cdot, \cdot)$, a distance metric between
  persistence diagrams, defined as
  \begin{equation}
    \bottleneck(\diagram, \diagram') = \inf_{\eta\colon \diagram \to \diagram'}\sup_{x\in{}\diagram}\|x-\eta(x)\|_\infty,
    \label{eq:Bottleneck distance}
  \end{equation}
  where $\eta\colon\diagram \to \diagram'$ denotes a bijection between the
  point sets of both diagrams, and $\|\cdot\|_\infty$ refers to the
  $\mathrm{L}_\infty$ metric between two points in $\reals^2$.
}
Using the preceding theorem from \autoref{sec:geometric}, we can bound
the topological activity and prove convergence in terms of topological
properties. \review{Specifically, for the \mbox{0}-dimensional
  persistence diagram of our input dataset at diffusion time~$t$, which
  we subsequently denote by $\diagram_{\X_{t}}$,
  we prove that the bottleneck distance
  $\bottleneck(\diagram_{\X_t}, \diagram_{\X_{t'}})$ to another time
  step~$t'$ is upper-bounded by the respective
diameters of the point clouds.}
\begin{theorem}\label{thm:persistence}
  Let $t \leq t'$ refer to two iterations of the
  diffusion condensation process \review{with $\X_t, \X_{t'}$ denoting
  their corresponding point clouds. If $\diam(\X_t) \geq
\diam(\X_{t^\prime})$,} then the persistence diagrams
  corresponding to $\X_t$ and $\X_{t'}$ satisfy
  \begin{equation}
    \bottleneck(\diagram_{\X_t}, \diagram_{\X_{t'}}) \leq \diam\left(\mathsf{X}_t\right).
  \end{equation}
  %
\end{theorem}
\begin{proof}
From Chazal et al.~\cite{Chazal09,
Chazal14}, we obtain
$\bottleneck(\diagram_{\X_t}, \diagram_{\X_{t'}}) \leq
2 \dgromovhausdorff(\X_t, \X_{t'})$, where $\dgromovhausdorff(\cdot,
  \cdot)$ denotes the Gromov--Hausdorff
distance.
According to M{\'e}moli~\cite[Proposition~5]{Memoli07}, we have
$\dgromovhausdorff(\X_t, \X_{t'}) \leq \nicefrac{1}{2} \max\{\diam(\X_t), \diam(\X_{t^\prime})\}$,
so we can simplify the bound to
$\bottleneck(\diagram_{\X_t}, \diagram_{\X_{t'}}) \leq \max\{\diam(\X_t), \diam(\X_{t^\prime})\}$.
As $\diam(\X_{t}) \geq \diam(\X_{t^\prime})$, we have $\bottleneck\left(\diagram_{\X_t}, \diagram_{\X_{t'}}\right) \leq \diam(\X_t)$. 
\end{proof}
\review{%
  Under the conditions of \cref{cor: spec_convergence}, i.e., for
  a large family of diffusion operators, we know that diffusion
  condensation converges to a point, thus implying $\lim_{t \to
  \infty}\diam(\X_t) = 0$.
  While we cannot guarantee that
  $\diam(\X_t) \geq \diam(\X_{t^\prime})$ for $t \leq t'$ holds in
  general~(in the setting of \cref{cor: spec_convergence},
  the diameter can increase; in the more restrictive setting of
  \cref{thm:diameter_conv}, diameters would also be non-increasing, but
  that theorem only applies to strictly pointwise positive kernels),
  we know that there exists a subsequence of
  condensation steps~$\{\widetilde{t}\}$ such that the diameter is non-increasing.
For this subsequence,} the bottleneck distance between
consecutive datasets, i.e., $\bottleneck(\diagram_{\X_{\widetilde{t}}},
\diagram_{\X_{\widetilde{t+1}}})$ also converges to~$0$. By contrast,
$\dgromovhausdorff(\X_{\widetilde{t}}, \X_{\widetilde{t+1}}) \geq \nicefrac{1}{2}
|\diam(\X_{\widetilde{t}}) - \diam(\X_{\widetilde{t+1}})|$~(the bound being tight in certain cases), implying that
the Bottleneck distance between consecutive time steps is never zero if
the diameter changes.
Since all point clouds are embedded into the same space, namely
$\reals^d$, all preceding statements apply with the Hausdorff distance
$\dhausdorff(\cdot, \cdot)$ replacing the Gromov--Hausdorff
distance~\cite{Chazal15}.
This distance has the advantage that we can easily evaluate it. We
require one auxiliary lemma to replace the diameter bound in the
previous proof.

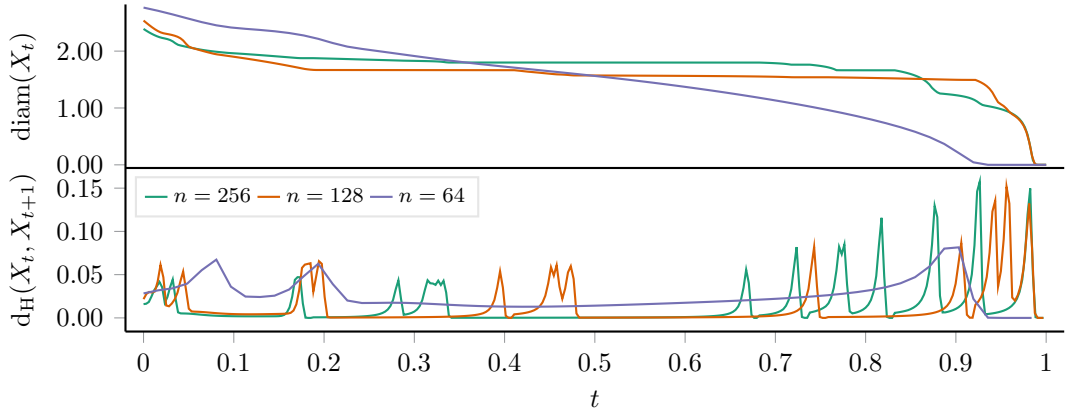
\begin{figure}[tbp]
  \centering
  \begin{tikzpicture}
    \begin{groupplot}[%
      lineplot,
      group style = {
        group size   = 1 by 2,
        vertical sep = 0pt,
      },
      group/x descriptions at = edge bottom,
    ]
      \nextgroupplot[
        ylabel = {$\diam(X_t)$},
        xtick  = \empty,
      ];

      \addplot+[%
        discard if not = {basename}{petals_n256}
      ] table[col sep=tab, x = t, y = diameter] {Data/Petals.csv};

      \addplot+[%
        discard if not = {basename}{petals_n128}
      ] table[col sep=tab, x = t, y = diameter] {Data/Petals.csv};

      \addplot+[%
        discard if not = {basename}{petals_n64}
      ] table[col sep=tab, x = t, y = diameter] {Data/Petals.csv};



      \nextgroupplot[
        ylabel           = {$\dhausdorff(X_t, X_{t+1})$},
        enlarge y limits = 0.1,
        legend style      = {%
          draw         = lightgrey,
          font         = \fontsize{8}{9}\selectfont,
          at           = {(0.01, 0.95)},
          anchor       = north west,
        },
      ];

      \addplot+[%
        discard if not = {basename}{petals_n256}
      ] table[col sep=tab, x = t, y = hausdorff_distance] {Data/Petals.csv};

      \addplot+[%
        discard if not = {basename}{petals_n128}
      ] table[col sep=tab, x = t, y = hausdorff_distance] {Data/Petals.csv};

      \addplot+[%
        discard if not = {basename}{petals_n64}
      ] table[col sep=tab, x = t, y = hausdorff_distance] {Data/Petals.csv};

      \legend{
        $n = 256$,
        $n = 128$,
        $n = 64$,
      }

    \end{groupplot}
  \end{tikzpicture}
  \caption{%
    \review2{Empirical convergence behavior of point cloud diameters~\review{%
      (for the ``petals'' dataset at various sample sizes~$n$)
    }
    and
    the Hausdorff distance between consecutive steps $i, i+1$ of the
    condensation process~(\review{used} as a proxy for the bottleneck distance).
    Convergence behavior with respect to the
    Hausdorff distance is not uniform and characterized by some
    ``jumps'', indicating that the datasets change considerably between
    certain time steps, before achieving a stable configuration.}
  }
  \label{fig:Empirical convergence}
\end{figure}

\begin{lemma}
  Let $\mathsf{X}, \mathsf{Y}$ be subsets of the same metric space, e.g., $\reals^d$,
  with $\conv(\mathsf{Y}) \subsetneq \conv(\mathsf{X})$. Then
  $\dhausdorff(\mathsf{X}, \mathsf{Y}) \leq \diam(\mathsf{X})$.
\end{lemma}
\begin{proof}
  The Hausdorff distance is the smallest $r$-thickening required such
  that both $\mathsf{X}$ and $\mathsf{Y}$ become subsets of each other, \review{%
  i.e.,
  $\dhausdorff(\mathsf{X}, \mathsf{Y}) := \inf \{ r > 0 \mid \mathsf{Y} \subseteq \mathsf{X}^{(r)} \text { and } \mathsf{X} \subseteq \mathsf{Y}^{(r)} \}$.
  }
  Since $\conv(\mathsf{Y}) \subsetneq \conv(\mathsf{X})$, we have $r
  \leq \diam(\mathsf{X})$.
\end{proof}
As a consequence of this lemma and the preceding proof, we obtain
a bound in terms of the Hausdorff distance and the diameter. For $t \leq
t'$, we have
\begin{equation*}
  \bottleneck(\diagram_{\X_t}, \diagram_{\X_{t'}}) \leq 2 \dhausdorff(\X_t, \X_{t'}) \leq \diam(\X_t).
\end{equation*}
\cref{fig:Empirical convergence} shows empirical convergence behavior
between consecutive time steps, illustrating how different condensation
processes are characterized by different diameter shrinkages.

\subsection{Hierarchical clustering}\label{sec:clustering}

\reviewtwo{%
In contrast to existing work on multiscale diffusion-based
clustering~\cite{murphy22multiscale}, diffusion condensation changes the
underlying geometric--topological structure of the data to extract
hierarchical information. A topological perspective helps us elucidate connections
to hierarchical clustering, a clustering method based on measuring
dissimilarities between clusters via \emph{linkage methods}.
}
While there are many linkage methods for measuring the association
between clusters in agglomerative hierarchical clustering, we focus on
the centroid method as it is the most relevant to diffusion
condensation. Agglomerative clustering, and the centroid method
specifically, is widely applied in
phylogeny~\cite{durbin1998biological}, sequence
alignment~\cite{katoh_mafft_2002}, and analysis of other
types of data~\cite{kaufman2009finding}. In the centroid method, the
distance between any two clusters $a$ and $b$ is defined as the distance
between the centroids of the clusters. There are two natural definitions
of Euclidean centroids, leading to the unweighted
pair group method with centroid mean
(UPGMC)~\cite{sokal_statistical_1958} and to the weighted version
(WPGMC), also known as median linkage hierarchical
clustering~\cite{gower_comparison_1967}. The unweighted centroid
$C_{\textsc{UPGMC}}$ is the centroid of all points in the cluster: 
\begin{equation}
C_{\textsc{UPGMC}}(a) = \frac{1}{|a|} \sum_{x \in a} x .
\end{equation}
In contrast, the weighted version depends on the parent clusters:
suppose cluster $a$ is formed by the merging of clusters $b$ and $c$,
then the centroid of $a$ is defined as:
\begin{equation}
    C_{\textsc{WPGMC}}(a) = \frac{C_{\textsc{WPGMC}}(b) + C_{\textsc{WPGMC}}(c)}{2}.
\end{equation}
In either case, the distance between two clusters $a$ and $b$ is defined
as the squared Euclidean distance between their centroids, denoted by
$D(a,b) := \|C(a) - C(b)\|^2$. This algorithm is detailed in
\cref{alg:clustering}~(\review{with $\oplus$ referring to
sequence concatenation}).
%
\begin{algorithm}
\caption{Centroid Hierarchical Agglomerative Clustering}
\label{alg:clustering}
\begin{algorithmic}[1]
\STATE{Input: set of points $\X_0$}
\STATE{Output: the set of clusters at each level $(L_0, L_1, \ldots, L_{N-1})$}
\STATE $L_0 \gets \left \{\{x(1)\}, \{x(2)\}, \ldots, \{x(N)\} \right \}$ \COMMENT{Initially, every point is its own cluster}
\FOR{$t \in \{1,\dotsc,N-1\}$}
    \STATE $a^*, b^* \gets \argmin_{(a, b) \in (L_{t-1})^2} D(a,b)$ s.t. $a \neq b$ \COMMENT{Find centroids to merge}
    \STATE $L_t \gets (L_{t-1} \setminus a^*) \setminus b^* \oplus (a^* \cup b^*)$ \COMMENT{Add new cluster with $a^*, b^*$ merged}
\ENDFOR
\end{algorithmic}
\end{algorithm}
%
For a given dataset, the UPGMC and WPGMC algorithms give a unique
sequence of merges, provided that at each iteration there exists a unique
choice of \review{centroids} $a^*$ and $b^*$ that achieve the minimum
distance between clusters. These methods are similar to diffusion
condensation, and in certain situations, equivalent; our next theorem
makes this more precise.

\begin{theorem}\label{thm:cluster_converge}
Let $\zeta = 0$, $\epsilon_t = \min_{x,y \in \Xt} \| x - y\|_2 > 0$, and
$k_t(x,y)$ be the box kernel in \cref{def:box}. In this case, the
diffusion condensation produces equivalent topological features than
centroid agglomerative clustering~(UPGMC), in both diffusion homology,
and persistent homology~(for $\zeta = 0$), i.e., $\vietoris_0(\X, t) = L_t$
for all $t$. Further if $\zeta : 0 < \zeta < \epsilon_t$, then diffusion
condensation is \review{similarly equivalent in both condensation
homology and persistent homology} to median linkage agglomerative clustering (WPGMC).
\end{theorem}

\begin{proof}
\review{
  To show the equivalence of diffusion condensation to UPGMC algorithms, we show that
  \begin{inparaenum}[(i)]
    \item centroids in the UPGMC algorithm correspond to points in the
      diffusion condensation algorithm, and
    \item the same clusters---represented by their respective
      centroids---are merged in each iteration, i.e. $\vietoris_0(\X, t)$ is a representation of the hierarchy of UPGMC at iteration $t$.
  \end{inparaenum}}

\review{
\begin{enumerate}[(i)]
    \item We show the first claim by induction. For the condensation algorithm at $t=0$, the claim is trivially true, since all points are singleton and therefore centroids. By the induction hypothesis, all points are centroids at time $t$, and we show it still holds at time $t+1$. Without loss of generality, we assume that only a single pair
      of points achieves the minimal pairwise distance at time $t$, say $(x_t(k),x_t(l))$.\footnote{%
      This is equivalent to
      assuming that all pairwise distances in the current diffusion
      condensation step are unique. Said assumption also ensures that the
      selection of $a^\ast$ and $b^\ast$ in \cref{alg:clustering} is
      unique, so it is a useful requirement. It does \emph{not} decrease the
      generality of our argumentation~(in fact, a consistent ordering of merges can
      always be achieved), but it simplifies notation and discussion.
      } Since $\epsilon_t = \min_{x,y \in \Xt} \| x - y\|_2 > 0$, and by construction of the box kernel $\mK_t(k,l) = \mK_t(l,k) =1$ and zero otherwise. Thus $x_{t+1}(k) = \mP_t(k,\cdot)\Xt = \mP_t(l,\cdot)\Xt = x_{t+1}(l)$, hence, only the two points with minimum distance will be merged at their midpoint (centroids), creating a new centroid. Since $\zeta = 0$, this will create a sequence of merges. 
      \item Just like UPGMC, only centroids with minimal distance are merged at every iteration. For each merge in the condensation algorithm, a tuple of the form $(0, t)$ is created in the condensation homology persistence diagram~(and the respective
      pairs are created in its persistence pairing).
      Therefore, points in the diffusion condensation process are equivalent
      to centroids in the UPGMC algorithm and the same merges happen in each
      iteration.
      Hence, $\vietoris_0(\X, t)$, the Vietoris--Rips
      complex of $\Xt$ at scale~$0$, is a representation of the hierarchy of
      UPGMC at iteration~$t$. 
\end{enumerate}
For the setting of $\zeta : 0 < \zeta < \epsilon_t$, the only difference is in the setting of the merge threshold. The proof follows the same logic as the previous theorem (assuming again pairwise distances in each step
are unique) except that the centroid locations are updated as the average of two
points, instead of a weighted average of all points in the two clusters, hence the equivalence with WPGMC.}
\end{proof}


\begin{remark}
  \review{With the conditions of \cref{thm:cluster_converge},} the theorem implies that diffusion condensation converges to a point in $N - 1$ iterations, as each iteration reduces the number of unique point locations by one. Extending this logic to more general settings (where the number of unique points might not strictly decrease in each iteration) is not trivial and is left to future work.
\end{remark}
\begin{remark} This theorem motivates interpreting diffusion
  condensation as a soft hierarchical clustering method, particularly
  with other kernels and in situations where the general
  position assumption does not hold. When points are equally spaced and
  not naturally clusterable, we find diffusion condensation
  more appealing: for instance, consider the corners of
  a $k$-dimensional simplex, with all distances between points being
  equal. The only two ``sensible'' clusterings are $k$ clusters of
  single points, or one cluster with $k$ points. Performing
  agglomerative clustering on this dataset will result in an arbitrary
  binary tree over the data, where all levels of the tree result in
  meaningless clusters. Diffusion condensation with any radial kernel,
  $\epsilon$ schedule, and merge threshold will result in exactly these
  two clusterings.
\end{remark}
\begin{remark}
  \review{This soft clustering interpretation also hints at
  a convergence result with potentially tighter bounds for general
kernels. The geometric results in \cref{sec:geometric} rely on
a pointwise lower bound of the kernel, this can lead to pessimistic
convergence results on kernels similar to the box kernel (for example
consider the $\alpha$-decay kernel with large $\alpha$), which act more
like hierarchical clustering but have poor tail bounds. An interesting
future direction would be to explore geometric convergence for general
kernels in terms of the number of unique points rather than the diameter
following the line of reasoning in \cref{thm:cluster_converge}.}
\end{remark}

\section{Discussion}


Diffusion condensation is a process that alternates between computing a data diffusion operator and applying the operator back on the data to gradually eliminate variation. In this paper, we analyzed the diffusion condensation process from two main perspectives -- its convergence and the evolution of its shape through condensation steps.  We found conditions guaranteeing the convergence of the process using both geometric and spectral arguments. The geometric argument shows that the convex hull of each iteration of data after condensation shrinks in comparison to the previous iteration. The spectral argument reasons that the second largest eigenvalues of the data graph bounds the result of any function multiplied by the diffusion operator. Our spectral results are of particular interest since they are valid for a broad family of diffusion operators creating a time-inhomogeneous process.

Further, we used and extended tools from topological data analysis to
characterize the evolution of the shape of the datasets during the
condensation process. In particular, we defined the condensation homology filtration that operates on the data manifold, and studied the resulting condensation homology. This provides us with a summary of the topological features during the entire process. Since the process is guaranteed to converge, the filtration will sweep through the different resolutions of the data, hence providing meaningful details. With the persistent diffusion homology, we studied the topological features for a given condensation step, resulting in snapshots of topological characteristics of the process.
Furthermore, we provided experiments showcasing the relevance of our analysis, specifically comparing the condensation and persistent homologies, and the usage of condensation for clustering purposes. 

We also showed that instances of diffusion condensation with the box kernel are equivalent to hierarchical clustering algorithms. \review{In future work we would like to extend this equivalence result to other ``softer'' kernels. This could potentially give a tighter convergence bound dependent on the concentration of a kernel and the number of points rather than its tail, which can lead to pessimistic bounds on k-nearest-neighbor random walks. Additionally, we would like to extend}
the definition of the intrinsic condensation filtration to
multidimensional filtrations, for instance by identifying the cycles or
considering path probabilities defined by the diffusion kernel.

%% file: supplement_text.tex
\section[Proof of lemma 3.3]{Proof of \cref{lem:geo_general_conv}}\label{sm:geo_general} 
  %
  %
We shall show that no point from $C \setminus \X$ can be extremal. Take an arbitrary point $v \in C \setminus \X$.
By definition of~$C$, $v$ can be written as a convex combination of all points, i.e.,
\begin{equation*}
    v = \sum_{i=1}^{N} \alpha_i x(i) = \alpha_1 x(1) + (1 - \alpha_1) \sum_{i=2}^{N} \frac{\alpha_i}{1 - \alpha_1} x(i).
\end{equation*}
In particular, since $v \notin \X$, each $\alpha_i$ satisfies $\alpha_i < 1$. Let $w := \sum_{i=2}^{N} \frac{\alpha_i}{1 - \alpha_1} x(i)$. This is a convex combination using points $x(2), \dots, x(N)$, so we may express~$v$ as $v = \alpha_1 x(1) + (1- \alpha_1) w$. Thus, $v$~can be placed on a line segment between two points of the polytope and it is neither the start point nor the end point. As a consequence, the point~$v$ is not extremal. \hfill \proofbox

\section[Proof of Lemma 3.4]{Proof of \cref{lemma:shrinkrate}}\label{sm:shrink}
   
\begin{proof}
    First we show the upper bound for TV distance between any two rows of matrix $\mP_t$. The $(i,j)$ entry of $\mP_t$ is given by
    \begin{equation*}
        \mP_t(i,j) = \frac{\Kt(i,j)}{\sum_j {\Kt(i,j)}},
    \end{equation*}
    where $\Kt$ is constructed using some kernel, and $
        1 \ge \Kt(i,j) \geq \delta > 0$.
    Thus a lower bound for $\Pt(i,j)$ is 
     \begin{equation*}
        \Pt(i,j) \ge \frac{\delta}{N}
    \end{equation*}
    
    The TV distance between any two rows of $\Pt$ 
    \begin{align*}
         d_{TV}(\Pt(i,\cdot), \Pt(j,\cdot)) &= \frac{1}{2} \sum_k | \Pt(i,k) - \Pt(j,k) |  \\
         &=\frac{1}{2} \sum_k \Pt(i,k) + \Pt(j,k) - 2 \min \{\Pt(i,k), \Pt(j,k)\}\\
         &\le 1 - \sum_k\min \{\Pt(i,k)\} \\
            &\le 1-\delta.\\
    \end{align*}
    
    After step $t$, two transformed data points:
  \begin{align*}
      x_{t+1}(i) &= \Pt(i,\cdot) \Xt \\
      x_{t+1}(j) &= \Pt(j,\cdot) \Xt.
  \end{align*}
  
  Consider a pair of random variables $Z_1$ and $Z_2$ of $\Pt(i, \cdot)$ and $\Pt(j, \cdot)$ respectively, and the joint distribution $\xi$ of $(Z_1,Z_2)$ on $[N]\times[N]$. Then $\xi$ satisfies $\sum_{z_2 \in [N]} \xi(z_1, z_2) = \Pt(i,\cdot)$ and $\sum_{z_1 \in [N]} \xi(z_1, z_2) = \Pt(j,\cdot)$. 
  
The distance between two points after step $t$
  \begin{align*}
       \| x_{t+1}(i) - x_{t+1}(j)\|_2 & = \|(\Pt(i,\cdot)-\Pt(j,\cdot)) \Xt\|_2 =\| \sum_{z_1, z_2} \xi(z_1, z_2)(x_t(z_1) -x_t(z_2) )\|_2 \\
       & \leq\sum_{z_1, z_2} \xi(z_1, z_2) \|x_t(z_1) -x_t(z_2)\|_2 \\
       & = \sum_{z_1 \neq z_2} \xi(z_1, z_2) \|x_t(z_1) -x_t(z_2)\|_2 \\
       &\leq \diam(\X_t) \sum_{z_1 \neq z_2} \xi(z_1, z_2).
  \end{align*}
\textcolor{blue}{From the coupling lemma~\cite{aldous1983random}}, we can choose the optimal coupling $(Z_1, Z_2)$,\\ so that $\sum_{z_1 \neq z_2} \xi(z_1, z_2) = d_{TV}(\Pt(i,\cdot),\Pt(j,\cdot))$. Then we obtain the upper bound for the distance between any two points after step $t$,
  \begin{equation*}
      \| x_{t+1}(i) - x_{t+1}(j)\|_2 \leq  d_{TV}(\Pt(i,\cdot), \Pt(j,\cdot)) 
      \diam(\X_t), \quad \forall (i,j)\in N^2.
  \end{equation*}
  Thus, $\diam(\X_{t+1})  \leq (1-\delta) \diam(\X_t)$.
\end{proof}

\section[Proof of Lemma]{Proof of \cref{lemma: bound Hpf_Hf}}\label{sm:bound Hpf-Hf}
\label{supp_sec:proof}
\begin{proof} 

Recall that we can write
\begin{equation*}
    H_t(\Pt f) = \sum_{k=2}^N \lambda_{t,k}\langle f,\psi_{t,k} \rangle_{d_t}\psi_{t,k}.
\end{equation*}
Hence, we have\begin{equation*}
    \norm{H_t(\Pt f)}_{d_t}^2 = \sum_{k=2}^N \lambda_{t,k}^2 |\langle f,\psi_{t,k} \rangle_{d_t}|^2\leq \lambda_{t,2}^2 \sum_{k=2}^N |\langle f,\psi_{t,k} \rangle_{d_t}|^2 = \lambda_{t,2}^2\norm{H_t(f)}_{d_t}^2,
\end{equation*}
which concludes the proof.
\end{proof}

\section[Proof of lemma 4.2]{Proof of \cref{lemma:change_of_measure}}\label{sm:change_of_measure}
\begin{proof}
We start by proving the first inequality. For two time indices $s$ and $t$ we have
\begin{align*}
    \norm{f}^2_{d_t} &= \sum_{i=1}^N d_t(i)|f(x(i))|^2 = \sum_{i=1}^N\frac{d_t(i)}{d_s(i)}d_s(i)|f(x(i))|^2\\ &\leq \norm{d_t/d_s}_\infty\sum_{i=1}^Nd_s(i)|f(x(i))|^2 = \norm{d_t/d_s}_\infty\norm{f}^2_{d_s}.
\end{align*}
For the proof of the second inequality, we first need to note that $1/d_t(i) \leq 1$ for all $i\in\{1,\dotsc,N\}$. Indeed, since we have 
\begin{align*}
    d_t(i) = \sum_{i=1}^N \mK_t(i,j) = \mK_t(i,i) + \sum_{j\neq i}\mK_t(i,j) = 1 + \sum_{j\neq i}\mK_t(i,j) \geq 1.
\end{align*}

Now we can write
\begin{align}
    \norm{d_t/d_s}_\infty = \max_{i}\frac{\abs{d_t(i)}}{\abs{d_s(i)}} &\leq \max_{i}\frac{\abs{d_t(i)-d_s(i)}+\abs{d_s(i)}}{d_s(i)}\notag\\
    &\leq \max_{i}\abs{d_t(i)-d_s(i)}\, \max_{i}\frac{1}{d_s(i)} + 1\notag\\
    \label{inproof: inv_norm_bound}
    &\leq \max_{i}\abs{d_t(i)-d_s(i)} + 1\\
    &\leq \norm{d_t-d_s}_2+1,\notag
\end{align}
where the inequality \eqref{inproof: inv_norm_bound} is obtained by the fact that $1/d_t(i) \leq 1$. Thus, we can conclude that $ \norm{f}^2_{d_t} \leq  \norm{d_t/d_s}_\infty\norm{f}^2_{d_s} \leq (\norm{d_t-d_s}_2+1)\norm{f}^2_{d_s}$.
\end{proof}

\section[Proof of lemma 4.3]{Proof of \cref{lemma: bound Ls-Lt}}\label{sm: bound ls-lt}
\begin{proof}
Since the first eigenvalue of $\Pt$ is $\lambda_{t,1} = 1$, we have $L_t(\Pt f) = L_t(f)$ and we can write
\begin{equation*}
    \mP_t f =\langle f,\mathds{1} \rangle_{\pi_t} \mathds{1} + \sum_{k=2}^N \lambda_{t,k} \langle f,\psi_{t,k} \rangle_{d_t} \psi_{t,k}.
\end{equation*}
Substituting the previous expression in $L_s(\Pt f)$ yields
\begin{align*}
    L_s(\Pt f) &= \big\langle \langle f,\mathds{1} \rangle_{\pi_t} \mathds{1} + \sum_{k=2}^N \lambda_{t,k} \langle f,\psi_{t,k} \rangle_{d_t} \psi_{t,k} ,\mathds{1} \big\rangle_{\pi_s} \mathds{1}\\
    &=  \big[\langle f,\mathds{1} \rangle_{\pi_t}\langle \mathds{1},\mathds{1} \rangle_{\pi_s} + \sum_{k=2}^N \lambda_{t,k} \langle f,\psi_{t,k} \rangle_{d_t} \langle\psi_{t,k},\mathds{1}\rangle_{\pi_s}\big] \mathds{1}\\
    &= \norm{d_s}_1^{-1}\big[\norm{d_t}_1^{-1}\langle f,\mathds{1} \rangle_{d_t} \langle \mathds{1},\mathds{1} \rangle_{d_s} + \sum_{k=2}^N \lambda_{t,k} \langle f,\psi_{t,k} \rangle_{d_t} \langle\psi_{t,k},\mathds{1}\rangle_{d_s}\big] \mathds{1}\\
    &= \big[\norm{d_t}_1^{-1}\langle f,\mathds{1} \rangle_{d_t} +  \norm{d_s}_1^{-1}\sum_{k=2}^N \lambda_{t,k} \langle f,\psi_{t,k} \rangle_{d_t} \langle\psi_{t,k},\mathds{1}\rangle_{d_s}\big] \mathds{1}, 
\end{align*}
where the last equality is due to the fact that $ \langle \mathds{1},\mathds{1} \rangle_{d_s} = \norm{d_s}_1$. We already observed that $\norm{d_t}_1^{-1}\langle f,\mathds{1} \rangle_{d_t} = L_t(f) = L_t(\Pt f)$, therefore 
\begin{equation*}
    L_t(\Pt f) - L_s(\Pt f) =  -\norm{d_s}_1^{-1}\sum_{k=2}^N \lambda_{t,k}\langle f,\psi_{t,k} \rangle_{d_t}\langle\psi_{t,k},\mathds{1},\rangle_{d_s}\mathds{1}.
\end{equation*}
This last equation can be upper bounded by observing that $\norm{c\mathds{1}}_{d_t} = \abs{c}\norm{d_t}_1^{1/2}$, and $N\leq \norm{d_t}_1\leq N^2$. Indeed, we can write 
\begin{align}
    \norm{L_t(\Pt f) - L_s(\Pt f)}_{d_t} &= \norm{d_s}_1^{-1}\norm{d_t}_1^{1/2}\abs{\sum_{k=2}^N \lambda_{t,k}\langle f,\psi_{t,k} \rangle_{d_t}\langle\psi_{t,k},\mathds{1},\rangle_{d_s}}\notag\\
    &\leq N^{-1} N \abs{\sum_{k=2}^N \lambda_{t,k}\langle f,\psi_{t,k} \rangle_{d_t}\langle\psi_{t,k},\mathds{1},\rangle_{d_s}}\notag\\
    &\leq \sum_{k=2}^N \lambda_{t,k}|\langle f,\psi_{t,k} \rangle_{d_t}|\, |\langle\psi_{t,k},\mathds{1},\rangle_{d_s}| \notag\\
    &\leq \bigg[ \sum_{k=2}^N \lambda_{t,k}|\langle f,\psi_{t,k} \rangle_{d_t}|^2 \bigg]^{1/2}\bigg[ \sum_{k=2}^N |\langle\psi_{t,k},\mathds{1},\rangle_{d_s}|^2 \bigg]^{1/2}\notag\\
    &\leq  \lambda_{t,2} \,\norm{H_t(f)}_{d_t}\bigg[ \sum_{k=2}^N |\langle\psi_{t,k},\mathds{1},\rangle_{d_s}|^2 \bigg]^{1/2}\label{eq_sup:proof_lemma4.3_bound}.
\end{align}
To finish the proof, we need to bound the term in bracket
\begin{align*}
    \sum_{k=2}^N |\langle\psi_{t,k},\mathds{1}\rangle_{d_s}|^2 &=  \sum_{k=2}^N\Big|\sum_{i=1}^N\psi_{t,k}(i)d_s(i)\Big|^2 = \sum_{k=2}^N\Big|\sum_{i=1}^N\psi_{t,k}(i)[d_s(i)-d_t(i) + d_t(i)]\Big|^2\\
    &= \sum_{k=2}^N |\langle\psi_{t,k},d_s-d_t\rangle + \langle\psi_{t,k},\mathds{1}\rangle_{d_t}|^2 = \sum_{k=2}^N\abs{\langle\psi_{t,k},d_s-d_t\rangle}^2\\
    &\leq \sum_{k=2}^N \norm{\psi_{t,k}}_2^2\norm{d_s-d_t}_2^2\\
    &\leq \norm{d_s-d_t}_2^2\norm{1/d_t}_\infty\sum_{k=2}^N \norm{\psi_{t,k}}_{d_t}^2\leq \norm{d_s-d_t}_2^2\sum_{k=2}^N \norm{\psi_{t,k}}_{d_t}^2\\
    &\leq N\norm{d_s-d_t}_2^2.
\end{align*}
Finally, we conclude by substituting the previous bound in the inequality \ref{eq_sup:proof_lemma4.3_bound}.
\end{proof}

\section{Topological data analysis~(TDA)}\label{sm:TDA}

This section provides
a brief introduction to the most relevant concepts in the emerging field of
topological data analysis, namely
\begin{inparaenum}[(i)]
  \item simplicial homology,
  \item persistent homology, and
  \item their calculation in the context of point clouds.
\end{inparaenum}
We refer readers to Edelsbrunner and Harer~\cite{Edelsbrunner10} for
a comprehensive description of these topics.

\paragraph{Simplicial homology}
%
Simplicial homology refers to a way of assigning connectivity
information to topological objects, such as manifolds, which are
represented by simplicial complexes.
A simplicial complex~$\simplicialcomplex$ is a set of \emph{simplices}
of some dimensions. These may be considered as subsets of an index set,
with nomenclature typically referring to vertices~(dimension~$0$),
edges~(dimension~$1$), and triangles~(dimension~$2$).
The subsets of a simplex~$\sigma \in \simplicialcomplex$ are referred to
as its \emph{faces}, and every face~$\tau$ needs to satisfy $\tau \in
\simplicialcomplex$. Moreover, any non-empty intersection of two
simplices also needs to be part of the simplicial complex, i.e.,
$\sigma \cap \sigma' \neq \emptyset$ for $\sigma, \sigma' \in
\simplicialcomplex$ implies $\sigma \cap \sigma' \in \simplicialcomplex$.
Therefore, $\simplicialcomplex$ is ``closed under
calculating the faces of a simplex.''

\paragraph{Chain groups}
%
To characterize simplicial complexes, it is necessary to imbue them with
additional algebraic structures.
For a simplicial complex $\simplicialcomplex$, let
$\chaingroup{d}(\simplicialcomplex)$ be the vector space generated over
$\mathds{Z}_2$~(the field with two elements), also known as the
\emph{chain group in dimension $d$}. The elements of
$\chaingroup{d}(\simplicialcomplex)$ are the $d$-simplices in
$\simplicialcomplex$ and their formal sums, with coefficients in
$\mathds{Z}_2$.
For instance, $\sigma + \tau$ is an
element of the chain group, also called a \emph{simplicial chain}.
Addition is well-defined and easy to implement since
a simplex can only be present or absent over~$\mathds{Z}_2$
coefficients.
The use of chain groups lies in providing the underlying vector space
to formalize boundary calculations over a simplicial complex, which in
turn are required for defining connectivity.

\paragraph{Boundary homomorphism and homology groups}
%
Given a \mbox{$d$-simplex} $\sigma = \{v_0,\dotsc,v_d\} \in
\simplicialcomplex$, its boundary is defined in terms of the boundary
operator 
$\boundaryop{d}\colon\chaingroup{d}(\simplicialcomplex)\to\chaingroup{d-1}(\simplicialcomplex)$,
with
\begin{equation}
  \boundaryop{d}(\sigma) := \sum_{i=0}^{d}(v_0,\dotsc, v_{i-1},v_{i+1},\dotsc, v_d),
\end{equation}
i.e., we leave out every vertex~$v_i$ of the simplex once. This is a map
between chain groups, and since only sum operations are involved, it is
readily seen to be a homomorphism. By linearity, we can extend this
calculation to $\chaingroup{d}(\simplicialcomplex)$.
The boundary homomorphism gives us a way to precisely define
connectivity by means of calculating its \emph{kernel} and
\emph{image}. Notice that the kernel~$\ker\boundaryop{d}$ contains
all \mbox{$d$-dimensional} simplicial chains that do not have
a boundary.
Finally, the $d$th homology group $\homologygroup{d}(\simplicialcomplex)$ of $\simplicialcomplex$ is defined as
the \emph{quotient group} $\homologygroup{d}(\simplicialcomplex) := \ker\boundaryop{d} / \im\boundaryop{d+1}$.
It contains all topological features---represented using simplicial
chains---that have no boundary while also not being the boundary of
a higher-dimensional simplex. Colloquially, the homology group therefore
measures the ``holes'' in $\simplicialcomplex$.

\paragraph{Betti numbers}
%
The \emph{rank} of the
$d$th homology group is an important
invariant of a simplicial complex, known as the $d$th Betti number~$\betti{d}$,
i.e., $\betti{d}(\simplicialcomplex) := \rank\homologygroup{d}(\simplicialcomplex)$.
The sequence of Betti numbers $\betti{0},\dotsc,\betti{d}$ of
a $d$-dimensional simplicial complex is commonly used to discriminate
between manifolds.
For example, a $2$-sphere has Betti numbers $(1,0,1)$, while a $2$-torus
has Betti numbers $(1,2,1)$.
Betti numbers are limited in expressivity when dealing with real-world
data sets because they are highly dependent on a specific choice of
simplicial complex~$\simplicialcomplex$. This limitation prompted the
development of persistent homology.

\paragraph{Persistent homology}
%
Persistent homology is an extension of simplicial homology. At its core,
it employs \emph{filtrations} to imbue a simplicial
complex~$\simplicialcomplex$ with scale information, resulting in
multi-scale topological information. We assume the existence of
a function $f\colon\simplicialcomplex\to\reals$, which only attains
a finite number of function values
$f^{(0)} \leq f^{(1)} \leq \dotsc \leq f^{(m-1)} \leq f^{(m)}$.
This permits us to sort~$\simplicialcomplex$ according
to~$f$, for example by extending~$f$ linearly to higher-dimensional
simplices via $f(\sigma) := \max_{v \in \sigma} f(v)$, leading to
a nested sequence of simplicial complexes
\begin{equation}
  \emptyset = \simplicialcomplex^{(0)} \subseteq \simplicialcomplex^{(1)} \subseteq \dots \subseteq \simplicialcomplex^{(m-1)} \subseteq \simplicialcomplex^{(m)} = \simplicialcomplex,
  \label{eq:Filtration}
\end{equation}
where $\simplicialcomplex^{(i)} := \left\{\sigma \in K \mid f(\sigma) \leq
f^{(i)} \right\}$. Each of these simplicial complexes therefore only
contains those simplices whose function value is less than or equal
to the threshold.
In contrast to simplicial homology, the filtration is more expressive,
because it permits us to track changes. For instance, a topological
feature might be \emph{created}~(a new connected component might arise)
or \emph{destroyed}~(two connected components might merge into one), as
we pass from $\simplicialcomplex^{(i)}$ to
$\simplicialcomplex^{(i+1)}$.
Persistent homology provides a principled way of tracking topological
features, representing each one by a creation and destruction value
$(f^{(i)}, f^{(j)}) \in \reals^2$ based on the filtration function, with $i \leq j$.
In case a topological feature is still present at the end of the
filtration, we refer to the feature as being \emph{essential}. These
features are the ones that are counted for the Betti number calculation.
It is also possible to obtain only tuples with finite persistence
values, a process known as \emph{extended
persistence}~\cite{Cohen-Steiner09}, but we eschew this concept in this
work for reasons of computational complexity.
Every filtration induces an inclusion homomorphism between
$\simplicialcomplex^{(i)} \subseteq \simplicialcomplex^{(i+1)}$. The
respective boundary homomorphisms in turn induce a homomorphism between corresponding homology
groups of the simplicial complexes of the filtration. These are maps of the form $\mathfrak{i}_d^{(i,j)} \colon \homologygroup{d}(\simplicialcomplex_i) \to
\homologygroup{d}(\simplicialcomplex_j)$.
This family of homomorphisms now gives rise to a sequence of homology groups
\begin{equation}
    \homologygroup{d}\left(\simplicialcomplex^{(0)}\right)
    \xrightarrow{\mathfrak{i}_d^{(0,1)}}
    \homologygroup{d}\left(\simplicialcomplex^{(1)}\right)
    \xrightarrow{\mathfrak{i}_d^{(1,2)}}  \dots
    \xrightarrow{\mathfrak{i}_d^{(m-2,m-1)}}
    \homologygroup{d}\left(\simplicialcomplex^{(m-1)}\right)\xrightarrow{\mathfrak{i}_d^{(m-1,m)}}\homologygroup{d}\left(\simplicialcomplex^{(m)}\right)
\end{equation}
for every dimension $d$, with $\homologygroup{d}\left(\simplicialcomplex^{(m)}\right)
= \homologygroup{d}\left(\simplicialcomplex\right)$.
For $i \leq j$, the $d$th persistent
homology group is defined as
\begin{equation}
  \persistenthomologygroup{d}{i,j} :=
  \ker\boundaryop{d}\left(\simplicialcomplex^{(i)}\right) / \left(
  \im\boundaryop{d+1}\left(\simplicialcomplex^{(j)}\right)\cap\ker\boundaryop{d}\left(\simplicialcomplex^{(i)}\right)\right).
\end{equation}
This group affords an intuitive description: it contains all homology classes \emph{created} in
$\simplicialcomplex^{(i)}$ that are \emph{still} present in
$\simplicialcomplex^{(j)}$.
We can now define a variant of the aforementioned Betti numbers, the $d$th
persistent Betti number, namely,
  $\persistentbetti{d}{i,j} := \rank \persistenthomologygroup{d}{i,j}$.
Since the persistent Betti numbers are indexed by $i$ and $j$, we can
consider persistent homology as a way of generating a sequences of Betti
numbers, as opposed to just calculating \emph{one} single number. This
sequence can be summarized in a \emph{persistence diagram}.

\paragraph{Persistence diagrams and pairings}
%
Given a filtration induced by a function $f\colon\simplicialcomplex\to\reals$ as described above, 
each tuple $(f^{(i)}, f^{(j)})$ is stored with multiplicity
\begin{equation}
  \mu_{i,j}^{(d)} := \left( \persistentbetti{d}{i,j-1} - \persistentbetti{d}{i,j} \right) - \left( \persistentbetti{d}{i-1,j-1} - \persistentbetti{d}{i-1,j} \right)
\end{equation}
in the $d$th persistence diagram $\diagram_d$, which is a multiset in
the extended Euclidean plane $\reals \times \reals \cup \{\infty\}$,
including all tuples of the form $(c, c)$ with infinite
multiplicity~(thus simplifying the calculation of distances).
Notice that for most pairs of indices, $\mu_{i,j}^{(d)} = 0$, so the
practical number of tuples is not quadratic in the number of function
values.
For a point $(x,y) \in \diagram_d$, we refer to the quantity
$\persistence(x,y) := |y-x|$ as its \emph{persistence}.
The idea of persistence arose in multiple contexts~\cite{Barannikov94,
Edelsbrunner02, Verri93}, but it is nowadays commonly used to analyze
functions on manifolds, where high persistence is seen to correspond to
\emph{features} of the function, while low persistence is typically
considered \emph{noise}.
Finally, we remark that persistence diagrams keep track of topological features by associating
them with tuples. In this perspective, the identity of topological
features, i.e., the pair of simplices involved in its creation or
destruction, is lost. A \emph{persistence pairing}~$\mathcal{P}$ rectifies this by
storing tuples of simplices~$(\sigma, \tau)$, where $\sigma$ is
a \mbox{$k$-simplex}~(the creator of the feature) and $\tau$ is
a \mbox{$(k+1)$-simplex}~(the destroyer of the feature). This pairing is
known to be unique in the sense that a simplex can either be a creator
or a destroyer, but not both~\cite{Edelsbrunner02}. The persistence
pairing and the persistence diagram are equivalent if and only if the
filtration is injective on the level of \mbox{$0$-simplices}, i.e.,
there are no duplicate filtration values.
In the intrinsic diffusion homology, we make use of the pairing to
track the hierarchical information created during the diffusion
condensation process.

\paragraph{Distances and stability}
%
Persistence diagrams can be endowed with a metric, known as the
\emph{bottleneck distance}. This metric is used to assess the stability
of persistence diagrams with respect to perturbations of their input
function. For two persistence diagrams $\diagram$ and $\diagram'$, their bottleneck distance is
calculated as
\begin{equation}
  \bottleneck(\diagram, \diagram') = \inf_{\eta\colon \diagram \to \diagram'}\sup_{x\in{}\diagram}\|x-\eta(x)\|_\infty,
\end{equation}\label{eq:supp:Bottleneck distance}
where $\eta\colon\diagram \to \diagram'$ denotes a bijection between the
point sets of both diagrams, and $\|\cdot\|_\infty$ refers to the
$\mathrm{L}_\infty$ metric between two points in $\reals^2$.
Calculating \eqref{eq:supp:Bottleneck distance} requires solving an optimal
assignment problem; recent work~\cite{Kerber17} discusses efficient
approximation strategies.
The primary appeal of the bottleneck distance is that it can be related
to the Hausdorff distance, thus building a bridge between geometry and
topology. A seminal stability theorem~\cite{Cohen-Steiner07} states that 
distances between persistence diagrams are bounded by the distance of
the functions that give rise to them: given a simplicial
complex~$\simplicialcomplex$ and two monotonic functions
$f,g\colon\simplicialcomplex\to\reals$, their corresponding persistence
diagrams $\diagram_f$ and $\diagram_g$ satisfy
\begin{equation}
  \bottleneck(\diagram_f, \diagram_g) \leq \|f-g\|_\infty,
  \label{eq:Bottleneck stability}
\end{equation}
where $\|f-g\|_\infty$ refers to the Hausdorff between the two
functions. In \cref{sec:topological}, we make use of \cref{eq:Bottleneck stability} and
a more generic stability bound when we characterize diffusion condensation in topological terms.

\paragraph{Vietoris--Rips complexes}
%
The formulation of persistent homology hinges on the generation of
a simplicial complex. We will use the Vietoris--Rips
complex, a classical construction~\cite{Vietoris27} that requires
a distance threshold $\delta$\footnote{%
  Usually, this threshold is referred to as $\epsilon$ in the TDA
  literature. We refrain from this in order to avoid confusing it
  with the kernel smoothing parameter.
} and a metric $\dist(\cdot, \cdot)$ such as the Euclidean distance.
The Vietoris--Rips complex at scale~$\delta$ of an input data set is defined as 
  $\vietoris_{\delta}\left(\mathsf{X}\right) := \{ \sigma \subseteq \mathsf{X} \mid \dist(x(i),
  x_(j)) \leq \delta \text{ for all } x(i), x(j) \in \sigma \}$,
i.e., $\vietoris_{\delta}\left(\mathsf{X}\right)$ contains all subsets
of the input space whose pairwise distances are less than or equal
to~$\delta$.
In this formulation, each simplex of $\vietoris_{\delta}$ is assigned
a weight according to the maximum distance of its vertices, leading to
$\mathrm{w}(\sigma) := \max_{\{x(i), x(j)\} \subseteq \sigma} \dist(x(i),
x(j))$ and $\mathrm{w}(\tau) = 0$ for \mbox{0-simplices}. Other weight
assignment strategies are also possible, but this distance-based
assignment enjoys stability properties~\cite{Chazal14}, similar
to~\eqref{eq:Bottleneck stability}. Letting $\vietoris_{\delta}\left(\mathsf{X}\right)$
and $\vietoris_{\delta}\left(\mathsf{Y}\right)$ refer to the
Vietoris--Rips complexes of two spaces $\mathsf{X}, \mathsf{Y}$, their
corresponding persistence diagrams $\diagram_{\mathsf{X}},
\diagram_{\mathsf{Y}}$ satisfy
\begin{equation}
  \bottleneck(\diagram_{\mathsf{X}}, \diagram_{\mathsf{Y}}) \leq 2 \dgromovhausdorff(\mathsf{X}, \mathsf{Y}),
  \label{eq:Bottleneck stability GH}
\end{equation}
with $\dgromovhausdorff(\cdot, \cdot)$ denoting the Gromov--Hausdorff
distance. This bound, originally due to Chazal et al.~\cite{Chazal09,
Chazal14}, is useful in relating geometrical and topological
properties of the diffusion condensation process in \cref{thm:persistence}.

%% file: main.bbl
\begin{thebibliography}{10}

\bibitem{aldous1983random}
{\sc D.~Aldous}, {\em Random walks on finite groups and rapidly mixing markov
  chains}, in S{\'e}minaire de Probabilit{\'e}s XVII 1981/82, Springer, 1983,
  pp.~243--297.

\bibitem{Barannikov94}
{\sc S.~A. Barannikov}, {\em The framed {M}orse complex and its invariants},
  Advances in Soviet Mathematics, 21 (1994), pp.~93--115.

\bibitem{brugnone_coarse_2019}
{\sc N.~Brugnone, A.~Gonopolskiy, M.~W. Moyle, M.~Kuchroo, D.~{van Dijk}, K.~R.
  Moon, D.~{Colon-Ramos}, G.~Wolf, M.~J. Hirn, and S.~Krishnaswamy}, {\em
  Coarse {{Graining}} of {{Data}} via {{Inhomogeneous Diffusion
  Condensation}}}, 2019 IEEE International Conference on Big Data,  (2019),
  pp.~2624--2633.

\bibitem{Carlsson09a}
{\sc G.~Carlsson, V.~de~Silva, and D.~Morozov}, {\em Zigzag persistent homology
  and real-valued functions}, in Proceedings of the Annual Symposium on
  Computational Geometry, 2009, pp.~247--256.

\bibitem{Chazal09}
{\sc F.~Chazal, D.~Cohen-Steiner, M.~Glisse, L.~J. Guibas, and S.~Y. Oudot},
  {\em Proximity of persistence modules and their diagrams}, in Proceedings of
  the Twenty-Fifth Annual Symposium on Computational Geometry, Association for
  Computing Machinery, 2009, pp.~237--246.

\bibitem{Chazal14}
{\sc F.~Chazal, V.~de~Silva, and S.~Oudot}, {\em Persistence stability for
  geometric complexes}, Geometriae Dedicata, 173 (2014), pp.~193--214.

\bibitem{Chazal15}
{\sc F.~Chazal, B.~Fasy, F.~Lecci, B.~Michel, A.~Rinaldo, and L.~Wasserman},
  {\em Subsampling methods for persistent homology}, in Proceedings of the 32nd
  International Conference on Machine Learning, 2015, pp.~2143--2151.

\bibitem{cheng1995mean}
{\sc Y.~Cheng}, {\em Mean shift, mode seeking, and clustering}, IEEE
  transactions on pattern analysis and machine intelligence, 17 (1995),
  pp.~790--799.

\bibitem{Cohen-Steiner07}
{\sc D.~Cohen-Steiner, H.~Edelsbrunner, and J.~Harer}, {\em Stability of
  persistence diagrams}, Discrete {\&} Computational Geometry, 37 (2007),
  pp.~103--120.

\bibitem{Cohen-Steiner09}
{\sc D.~Cohen-Steiner, H.~Edelsbrunner, and J.~Harer}, {\em Extending
  persistence using {P}oincar{\'e} and {L}efschetz duality}, Foundations of
  Computational Mathematics, 9 (2009), pp.~79--103.

\bibitem{coifman_diffusion_2006}
{\sc R.~R. Coifman and S.~Lafon}, {\em Diffusion maps}, Applied and
  Computational Harmonic Analysis, 21 (2006), pp.~5--30.

\bibitem{diaconis1991geometric}
{\sc P.~Diaconis and D.~Stroock}, {\em Geometric bounds for eigenvalues of
  {M}arkov chains}, The Annals of Applied Probability,  (1991), pp.~36--61.

\bibitem{durbin1998biological}
{\sc R.~Durbin, S.~R. Eddy, A.~Krogh, and G.~Mitchison}, {\em Biological
  sequence analysis: probabilistic models of proteins and nucleic acids},
  Cambridge university press, 1998.

\bibitem{Edelsbrunner10}
{\sc H.~Edelsbrunner and J.~Harer}, {\em Computational topology: {A}n
  introduction}, American Mathematical Society, Providence, RI, USA, 2010.

\bibitem{Edelsbrunner02}
{\sc H.~Edelsbrunner, D.~Letscher, and A.~J. Zomorodian}, {\em Topological
  persistence and simplification}, Discrete {\&} Computational Geometry, 28
  (2002), pp.~511--533.

\bibitem{Feragen15a}
{\sc A.~Feragen, F.~Lauze, and S.~Hauberg}, {\em Geodesic exponential kernels:
  When curvature and linearity conflict}, in Proceedings of the IEEE Conference
  on Computer Vision and Pattern Recognition~(CVPR), 2015.

\bibitem{fukunaga1975estimation}
{\sc K.~Fukunaga and L.~Hostetler}, {\em The estimation of the gradient of a
  density function, with applications in pattern recognition}, IEEE
  Transactions on information theory, 21 (1975), pp.~32--40.

\bibitem{Ghrist08a}
{\sc R.~Ghrist}, {\em Barcodes: {T}he persistent topology of data}, Bulletin of
  the American Mathematical Society, 45 (2008), pp.~61--75.

\bibitem{gower_comparison_1967}
{\sc J.~C. Gower}, {\em A comparison of some methods of cluster analysis},
  Biometrics, 23 (1967).

\bibitem{Hensel21}
{\sc F.~Hensel, M.~Moor, and B.~Rieck}, {\em A survey of topological machine
  learning methods}, Frontiers in Artificial Intelligence, 4 (2021).

\bibitem{katoh_mafft_2002}
{\sc K.~Katoh}, {\em {{MAFFT}}: A novel method for rapid multiple sequence
  alignment based on fast {{Fourier}} transform}, Nucleic Acids Research, 30
  (2002), pp.~3059--3066.

\bibitem{kaufman2009finding}
{\sc L.~Kaufman and P.~J. Rousseeuw}, {\em Finding groups in data: an
  introduction to cluster analysis}, vol.~344, John Wiley \& Sons, 2009.

\bibitem{Kerber17}
{\sc M.~Kerber, D.~Morozov, and A.~Nigmetov}, {\em Geometry helps to compare
  persistence diagrams}, ACM Journal of Experimental Algorithmics, 22 (2017).

\bibitem{kuchroo_topological_2021}
{\sc M.~Kuchroo, M.~DiStasio, E.~Calapkulu, M.~Ige, L.~Zhang, A.~H. Sheth,
  M.~Menon, Y.~Xing, S.~Gigante, J.~Huang, R.~M. Dhodapkar, B.~Rieck, G.~Wolf,
  S.~Krishnaswamy, and B.~P. Hafler}, {\em Topological analysis of single-cell
  data reveals shared glial landscape of macular degeneration and
  neurodegenerative diseases}, bioRxiv,  (2021).

\bibitem{kuchroo_multiscale_2020}
{\sc M.~Kuchroo, J.~Huang, P.~Wong, J.-C. Grenier, D.~Shung, A.~Tong, C.~Lucas,
  J.~Klein, D.~B. Burkhardt, S.~Gigante, A.~Godavarthi, B.~Rieck, B.~Israelow,
  M.~Simonov, T.~Mao, J.~E. Oh, J.~Silva, T.~Takahashi, C.~D. Odio,
  A.~Casanovas-Massana, J.~Fournier, {Yale IMPACT Team}, A.~Obaid, A.~Moore,
  A.~Lu-Culligan, A.~Nelson, A.~Brito, A.~Nunez, A.~Martin, A.~L. Wyllie,
  A.~Watkins, A.~Park, A.~Venkataraman, B.~Geng, C.~Kalinich, C.~B.~F. Vogels,
  C.~Harden, C.~Todeasa, C.~Jensen, D.~Kim, D.~McDonald, D.~Shepard,
  E.~Courchaine, E.~B. White, E.~Song, E.~Silva, E.~Kudo, G.~DeIuliis, H.~Wang,
  H.~Rahming, H.-J. Park, I.~Matos, I.~M. Ott, J.~Nouws, J.~Valdez, J.~Fauver,
  J.~Lim, K.-A. Rose, K.~Anastasio, K.~Brower, L.~Glick, L.~Sharma, L.~Sewanan,
  L.~Knaggs, M.~Minasyan, M.~Batsu, M.~Tokuyama, M.~C. Muenker, M.~Petrone,
  M.~Kuang, M.~Nakahata, M.~Campbell, M.~Linehan, M.~H. Askenase, M.~Simonov,
  M.~Smolgovsky, N.~D. Grubaugh, N.~Sonnert, N.~Naushad, P.~Vijayakumar, P.~Lu,
  R.~Earnest, R.~Martinello, R.~Herbst, R.~Datta, R.~Handoko, S.~Bermejo,
  S.~Lapidus, S.~Prophet, S.~Bickerton, S.~Velazquez, S.~Mohanty, T.~Alpert,
  T.~Rice, W.~Schulz, W.~Khoury-Hanold, X.~Peng, Y.~Yang, Y.~Cao, Y.~Strong,
  S.~Farhadian, C.~S. Dela~Cruz, A.~I. Ko, M.~J. Hirn, F.~P. Wilson, J.~G.
  Hussin, G.~Wolf, A.~Iwasaki, and S.~Krishnaswamy}, {\em Multiscale {PHATE}
  identifies multimodal signatures of {COVID}-19}, Nature Biotechnology,
  (2022).

\bibitem{lay2007convex}
{\sc S.~R. Lay}, {\em Convex sets and their applications}, Courier Corporation,
  2007.

\bibitem{Lesnick15}
{\sc M.~Lesnick and M.~Wright}, {\em Interactive visualization of {2-D}
  persistence modules}.
\newblock 2015, \url{https://arxiv.org/abs/1512.00180}.

\bibitem{maggioni2019learning}
{\sc M.~Maggioni and J.~M. Murphy}, {\em Learning by unsupervised nonlinear
  diffusion.}, Journal of Machine Learning Research, 20 (2019), pp.~1--56.

\bibitem{marshall2018time}
{\sc N.~F. Marshall and M.~J. Hirn}, {\em Time coupled diffusion maps}, Applied
  and Computational Harmonic Analysis, 45 (2018), pp.~709--728.

\bibitem{Memoli07}
{\sc F.~M{\'e}moli}, {\em On the use of {G}romov--{H}ausdorff distances for
  shape comparison}, in Eurographics Symposium on Point-Based Graphics,
  M.~Botsch, R.~Pajarola, B.~Chen, and M.~Zwicker, eds., The Eurographics
  Association, 2007.

\bibitem{moon_visualizing_2019}
{\sc K.~R. Moon, D.~{van Dijk}, Z.~Wang, S.~Gigante, D.~B. Burkhardt, W.~S.
  Chen, K.~Yim, A.~van~den Elzen, M.~J. Hirn, R.~R. Coifman, N.~B. Ivanova,
  G.~Wolf, and S.~Krishnaswamy}, {\em Visualizing structure and transitions in
  high-dimensional biological data}, Nature Biotechnology, 37 (2019),
  pp.~1482--1492.

\bibitem{moyle_structural_2021}
{\sc M.~W. Moyle, K.~M. Barnes, M.~Kuchroo, A.~Gonopolskiy, L.~H. Duncan,
  T.~Sengupta, L.~Shao, M.~Guo, A.~Santella, R.~Christensen, A.~Kumar, Y.~Wu,
  K.~R. Moon, G.~Wolf, S.~Krishnaswamy, Z.~Bao, H.~Shroff, W.~A. Mohler, and
  D.~A. Colón-Ramos}, {\em Structural and developmental principles of neuropil
  assembly in {C}. elegans}, Nature, 591 (2021), pp.~99--104.

\bibitem{murphy22multiscale}
{\sc J.~M. Murphy and S.~L. Polk}, {\em A multiscale environment for learning
  by diffusion}, Applied and Computational Harmonic Analysis, 57 (2022),
  pp.~58--100, \url{https://doi.org/10.1016/j.acha.2021.11.004}.

\bibitem{sokal_statistical_1958}
{\sc R.~R. Sokal and C.~D. Michener}, {\em A statistical method for evaluating
  systematic relationships}, University of Kansas Scientific Bulletin, 28
  (1958), pp.~1409--1438.

\bibitem{szlam2008regularization}
{\sc A.~D. Szlam, M.~Maggioni, and R.~R. Coifman}, {\em Regularization on
  graphs with function-adapted diffusion processes.}, Journal of Machine
  Learning Research, 9 (2008).

\bibitem{van2018recovering}
{\sc D.~Van~Dijk, R.~Sharma, J.~Nainys, K.~Yim, P.~Kathail, A.~J. Carr,
  C.~Burdziak, K.~R. Moon, C.~L. Chaffer, D.~Pattabiraman, et~al.}, {\em
  Recovering gene interactions from single-cell data using data diffusion},
  Cell, 174 (2018), pp.~716--729.

\bibitem{Verri93}
{\sc A.~Verri, C.~U. Uras, P.~Frosini, and M.~Ferri}, {\em On the use of size
  functions for shape analysis}, Biological Cybernetics, 70 (1993),
  pp.~99--107.

\bibitem{Vietoris27}
{\sc L.~Vietoris}, {\em \"{U}ber den h{\"o}heren {Z}usammenhang kompakter
  {R}{\"a}u\-me und eine {K}lasse von zusammenhangstreuen {A}bbildungen},
  Mathematische Annalen, 97 (1927), pp.~454--472.

\bibitem{von2007tutorial}
{\sc U.~Von~Luxburg}, {\em A tutorial on spectral clustering}, Statistics and
  computing, 17 (2007), pp.~395--416.

\end{thebibliography}
